\documentclass[nohyperref]{article}

\usepackage{microtype}
\usepackage{graphicx}
\usepackage{subfigure}
\usepackage{booktabs} 

\usepackage{hyperref}

\usepackage[accepted]{icml2023}

\usepackage{amsmath}
\usepackage{amssymb}
\usepackage{mathtools}
\usepackage{amsthm}
\usepackage{amssymb} 

\usepackage[capitalize,noabbrev]{cleveref}

\newtheoremstyle{slplain}
  {.5\baselineskip\@plus.2\baselineskip\@minus.2\baselineskip}
  {.5\baselineskip\@plus.2\baselineskip\@minus.2\baselineskip}
  {\slshape}
  {}
  {\bfseries}
  {.}
  { }
  {}

\theoremstyle{plain}
\newtheorem{theorem}{Theorem}[section]
\newtheorem{proposition}[theorem]{Proposition}

\newtheorem{corollary}[theorem]{Corollary}
\theoremstyle{definition}
\newtheorem{definition}[theorem]{Definition}

\theoremstyle{remark}

\theoremstyle{plain}
\newtheorem*{theorem*}{Theorem}
\newtheorem*{lemma*}{Lemma}
\newtheorem*{proposition*}{Proposition}
\newtheorem*{corollary*}{Corollary}

\usepackage[textsize=tiny]{todonotes}

\DeclareMathAlphabet{\altmathcal}{OMS}{cmsy}{m}{n} 

\newcommand{\essinf}[1]{ \underset{#1}{\operatorname{ess \ inf \ }}}
\newcommand{\esssup}[1]{ \underset{#1}{\operatorname{ess \ sup \ }}}
\newcommand{\argmin}[1]{\underset{#1}{\operatorname{arg \ min \ }}}

\newcommand{\sgn}{\text{sgn}}
\newcommand{\avg}{\text{avg}}

\newcommand{\bb}{\mathbb}

\newcommand{\indicator}[1]{\mathbbm 1_{#1}}

\newcommand{\abs}[1]{\left| #1 \right|}

\newcommand{\expect}[2]{\bb E_{#2}\left[ #1 \right]}

\renewcommand{\mathbf}{\boldsymbol}

\DeclareMathOperator{\supp}{supp}

\usepackage{enumitem}
\usepackage{bbm}
\usepackage{multirow}
\usepackage{tablefootnote}
\usepackage{bm}

\icmltitlerunning{Mixture Proportion Estimation Beyond Irreducibility}

\begin{document}

\twocolumn[
\icmltitle{Mixture Proportion Estimation Beyond Irreducibility}



\icmlsetsymbol{equal}{*}

\begin{icmlauthorlist}
\icmlauthor{Yilun Zhu}{eecs}
\icmlauthor{Aaron Fjeldsted}{psu}
\icmlauthor{Darren Holland}{afit}
\icmlauthor{George Landon}{cdv}
\icmlauthor{Azaree Lintereur}{psu}
\icmlauthor{Clayton Scott}{eecs,stats}

\end{icmlauthorlist}

\icmlaffiliation{eecs}{Department of Electrical Engineering and Computer Science, University of Michigan.}
\icmlaffiliation{stats}{Department of Statistics, University of Michigan}
\icmlaffiliation{psu}{Ken and Mary Alice Lindquist Department of Nuclear Engineering, Penn State University.}
\icmlaffiliation{afit}{Department of Engineering Physics, Air Force Institute of Technology.}
\icmlaffiliation{cdv}{School of Engineering and Computer Science, Cedarville University.}

\icmlcorrespondingauthor{Yilun Zhu}{allanzhu@umich.edu}
\icmlcorrespondingauthor{Clayton Scott}{clayscot@umich.edu}

\icmlkeywords{Machine Learning, Mixture Proportion Estimation, Weakly Supervised Learning} 

\vskip 0.3in
]



\printAffiliationsAndNotice{}  

\begin{abstract}
    The task of mixture proportion estimation (MPE) is to estimate the weight of a component distribution in a mixture, given observations from both the component and mixture. Previous work on MPE adopts the \emph{irreducibility} assumption, which ensures identifiablity of the mixture proportion. In this paper, we propose a more general sufficient condition that accommodates several settings of interest where irreducibility does not hold. We further present a resampling-based meta-algorithm that takes any existing MPE algorithm designed to work under irreducibility and adapts it to work under our more general condition. Our approach empirically exhibits improved estimation performance relative to baseline methods and to a recently proposed regrouping-based algorithm.
\end{abstract}

\section{Introduction}
\label{sec:intro}
Mixture proportion estimation (MPE) is the problem of estimating the weight of a component distribution in a mixture. Specifically, let $\kappa^* \in [0,1]$ and let $F$, $G$, and $H$ be probability distributions such that $F = (1-\kappa^*) G + \kappa^* H$. Given i.i.d. observations 
\begin{equation}
    \begin{split}
        X_H & := \left\{ x_1, x_2, \cdots, x_m \right\} \stackrel{iid}{\sim} H, \\ 
        X_F & := \left\{ x_{m+1}, x_{m+2}, \cdots, x_{m+n} \right\} \stackrel{iid}{\sim} F, 
    \end{split}
     \label{eq:setup}
\end{equation}
MPE is the problem of estimating $\kappa^*$. A typical application is given some labeled positive reviews $X_H$, estimate the proportion of positive comments about a
product among all comments $X_F$ \citep{gonzalez2017review}. MPE is also an important component in solving several domain adaptation and weakly supervised learning problems, such as learning from positive and unlabeled examples (LPUE) \citep{elkan2008learning, du2014analysis, kiryo2017positive}, learning with noisy labels \citep{lawrence2001estimating, natarajan2013learning, blanchard2016classification}, multi-instance learning \citep{zhang2001dd}, and anomaly detection \citep{sanderson2014class}.

If no assumptions are made on the unobserved component $G$, then $\kappa^*$ is not identifiable. \citet{blanchard2010semi} proposed the \emph{irreducibility} assumption on $G$ so that $\kappa^*$ becomes identifiable. Up to now, almost all MPE algorithms build upon the irreducibility assumption \citep{ blanchard2010semi, scott2015rate, blanchard2016classification, jain2016nonparametric, ramaswamy2016mixture, ivanov2020dedpul, bekker2020learning, garg2021mixture}, or some stricter conditions like non-overlapping support of component distributions \citep{elkan2008learning, du2014class}. However, as we discuss below, irreducibility can be violated in several applications, in which case the above methods produce statistically inconsistent estimates. As far as we know, \citet{yao2022rethinking}, discussed in Sec. \ref{sec:rempe}, is the first attempt to move beyond irreducibility. 

This work 
proposes a more general sufficient condition than irreducibility, and offers a practical algorithm for estimating $\kappa^*$ under this condition. We introduce a meta-algorithm that takes as input any MPE method that consistently estimates $\kappa^*$ under irreducibility, and removes the bias of that method whenever irreducibility does not hold but our more general sufficient condition does. Furthermore, even if our new sufficient condition is not satisfied, our meta-algorithm will not increase the bias of the underlying MPE method. We describe several applications and settings where our framework is relevant, and demonstrate the practical relevance of this framework through extensive experiments. 
Proofs and additional details can be found in the appendices.

\section{Problem Setup and Background}
\label{sec:setup}

Let $G$ and $H$ be probability measures on a measurable space $(\mathcal{X}, \mathfrak{S})$, and let $F$ be a mixture of $G$ and $H$
\begin{equation}
    F = (1-\kappa^*) G + \kappa^* H,
    \label{eq:mix}
\end{equation}
where $ 0 \leq \kappa^* \leq 1 $. 
With no assumptions on $G$, $\kappa^*$ is not uniquely determined by $F$ and $H$. 
For example, suppose $F = (1-\kappa^*) G + \kappa^* H$ for some $G$, and  take any $\delta \in [0, \kappa^*]$. Then $ F = (1-\kappa^* + \delta)G' + (\kappa^* - \delta)H, $ where $G' = (1-\kappa^* + \delta)^{-1} \left[ (1-\kappa^*) G + \delta H \right] $, has a different proportion on $H$ \citep{blanchard2010semi}. 


\subsection{Ground-Truth and Maximal Proportion}

To address the lack of identifiability, \citet{blanchard2010semi} introduced the so-called irreducibility assumption. We now recall this definition and related concepts. Throughout this work we assume that $F$, $G$ and $H$ have densities $f$, $g$ and $h$, defined w.r.t. a common dominating measure $\mu$. 
\begin{definition}[\citet{blanchard2010semi}]
    For any two probability distributions $F$ and $H$, define 
    \begin{equation*}
        \begin{split}
            \kappa(F|H) := \sup \{  \kappa \in [0, 1] \vert F &= (1-\kappa)G' + \kappa H, \\
            & \text{for some distribution } G'  \},
        \end{split}
    \end{equation*}
    the maximal proportion of $H$ in $F$.
    \label{def:maxprop}
\end{definition}
This quantity equals the infimum of the likelihood ratio:
\begin{proposition}[\citet{blanchard2010semi}] 
    It holds that
    \begin{equation} 
        \kappa(F|H) = \inf_{S \in \mathfrak{S}: H(S) > 0 } \frac{F(S)}{H(S)} = \essinf{x : h(x) > 0 } \frac{f(x)}{h(x)}.
        \label{eq:inf}
    \end{equation}
    \label{prop:inf}
\end{proposition}

By substituting $F = (1-\kappa^*) G + \kappa^* H$ into Eqn. \eqref{eq:inf}, we get that 
\begin{equation}
    \begin{split}
        \kappa(F|H) & = \inf_{S \in \mathfrak{S}: H(S) > 0 } \frac{F(S)}{H(S)} \\
                    & = \kappa^* + (1-\kappa^*) \inf_{S \in \mathfrak{S}: H(S) > 0 } \frac{G(S)}{H(S)} \\
                    & = \kappa^* + (1-\kappa^*) \kappa(G|H) .
    \end{split}
    \label{eq:reducible}
\end{equation}   
Since $\kappa(F|H)$ is identifiable from $F$ and $H$, the following assumption on $G$ ensures identifiability of $\kappa^*$.
\begin{definition}[\citet{blanchard2010semi}]
    We say that $G$ is \emph{irreducible} with respect to $H$ if $\kappa(G|H) = 0$. 
\end{definition}    

Thus, irreducibility means that there exists \emph{no} decomposition of the form: $ G = \gamma H + (1-\gamma)J'$, where $J'$ is some probability distribution and $ 0 < \gamma \leq 1 $. Under irreducibility, $\kappa^*$ is identifiable, and in particular, equals $\kappa(F|H)$.

\subsection{Latent Label Model}
We now consider another way of understanding irreducibility in terms of a latent label model. In particular, let $X$ and $Y \in \{0, 1\}$ be the random variables characterized by
\begin{enumerate}[label=(\alph*)]
    \item $(X,Y)$ are jointly distributed
    \item $P(Y=1) = \kappa^*$
    \item $P_{X|Y=0} = G \text{ and } P_{X|Y=1}=H  $.
\end{enumerate}
It follows from these assumptions that the marginal distribution of $X$ is $F$:
\[
P_X = \left(1 - \kappa^* \right) P_{X|Y=0} + \kappa^* P_{X|Y=1}  = F.
\]
We also take the conditional probability of $Y$ given $X$ to be defined via
\begin{equation}
    \begin{split}
        P(Y=1|X=x) = \begin{cases}
        \frac{\kappa^* h(x)}{f(x)}, & f(x) > 0, \\
        0, & \text{otherwise}.
        \end{cases}
    \end{split}
    \label{eq:cond}
\end{equation}

The latent label model is commonly used in the positive unlabeled (PU) learning literature \citep{bekker2020learning}. MPE is also called class prior/proportion estimation (CPE) in PU learning because $P(Y=1) = \kappa^*$. $Y$ may be viewed as a label indicating which component an observation from $F$ was drawn from. Going forward, we use this latent label model in addition to the original MPE notation.

\begin{proposition}
    Under the latent label model, 
    \begin{equation*}
        \esssup{x} P(Y=1|X=x) = \frac{\kappa^*}{\kappa(F|H)} \ ,
    \end{equation*}
    where $0/0 := 0$.
    \label{prop:post}
\end{proposition}

    


By definition, we know that $G$ is not irreducible with respect to $H$ iff $\kappa(G|H) > 0$. Combining Proposition \ref{prop:post} and Eqn. \eqref{eq:reducible}, we conclude that $\esssup{}  P(Y=1|X=x) < 1$ is equivalent to $\kappa(G|H)>0$.

\subsection{Violation of Irreducibility}
\label{subsec:vio}

Up to now, almost all MPE algorithms assume $G$ to be irreducible w.r.t. $H$ \citep{blanchard2010semi, blanchard2016classification, jain2016nonparametric,  ivanov2020dedpul}, or stricter conditions like the anchor set assumption \citep{scott2015rate, ramaswamy2016mixture}, or that $G$ and $H$ have disjoint supports \citep{elkan2008learning, du2014class}. These methods 
return an estimate of $\kappa(F|H)$ as the estimate of $\kappa^*$. If irreducibility does not hold and $\kappa^* < 1$, then $\kappa(F|H) > \kappa^*$. Even if these methods are consistent estimators of $\kappa(F|H)$, they are asymptotically biased and thus inconsistent estimators of $\kappa^*$.

A sufficient condition for irreducibility to hold is that the support of $H$ is not totally contained in the support of $G$. In a classification setting where $G$ and $H$ are the class-conditional distributions, this may be quite reasonable. It essentially assumes that each class has at least some subset of instances (with positive measure) that cannot possibly be confused with the other class. 
While irreducibility is reasonable in many classification-related tasks, there are also a number of important applications where it does not hold.
In this subsection we give three examples of applications where irreducibility is not satisfied.

\textbf{Ubiquitous Background.}
In gamma spectroscopy, we may view $H$ as the distribution of the energy of a gamma particle emitted by some source of interest (e.g., Cesium-137), and $G$ as the energy distribution of background radiation. Typically the background emits gamma particles with a wider range of energies than the source of interest does, and therefore its distribution has a wider support: $ \supp(H) \subset \supp(G)$, thus violating irreducibility. 
What's more, $G$ is usually unknown, because it varies according to the surrounding environment \citep{alamaniotis2013kernel}. 
The MPE problem is: given observations of the source spectrum $H$, which may be collected in a laboratory setting, and observations of $F$ in the wild, estimate $\kappa^*$. This quantity is important for nuclear threat detection and nuclear safeguard applications.


\textbf{Global Uncertainty.} In marketing, let $Y \in \left\{0, 1 \right\}$ denote whether a customer does ($Y=1$) or does not purchase a product \cite{fei2013heat}. Let $H$ be the distribution of a feature vector extracted from a customer who buys the product, and $G$ the distribution for those who do not. The MPE problem is: given data from past purchasers of a product ($H$), and from a target population ($F$), estimate the proportion $\kappa^*$ of $H$ in $F$. This quantity is called the \emph{transaction rate}, and is important for estimating the number of products to be sold. Irreducibility is likely to be violated here because, given a finite number of features, uncertainty about customers should remain bounded away from 1: $\forall x, P(Y=1|X=x) < 1 $. In other words, there is an $\epsilon > 0$ such that, for any feature vector of demographic information, the probability of buying the product is always $< 1-\epsilon$.

\textbf{Underreported Outcomes.} In public health, let $(X,Y,Z)$ be a jointly distributed triple, where $X$ is a feature vector, $Y \in \{0, 1\}$ denotes whether a person reports a health condition or not, and $Z \in \{0, 1\}$ indicates whether the person truly has the health condition.  Here, $H = P_{X|Y=1}, G = P_{X|Y=0}$ and $F = P_X$. The MPE problem is: given data from previous people who reported the condition ($H$), and from a target group ($F$), determine the \emph{prevalence} $\kappa^*$ of the condition for the target group. This helps estimate the amount of resources needed to address the health condition. Assume there are no false reports from those who do not have the medical condition: $\forall x, P(Y=1|Z=0, X=x) = 0$. Some medical conditions are underreported, such as smoking and intimate partner violence \citep{gorber2009accuracy,  shanmugam2021quantifying}. If the underreporting happens globally, meaning $e(x):= P(Y=1|Z=1, X=x)$ is bounded away from 1,  then $\esssup{}  P(Y=1|X=x) < 1$. This is because 
$P(Y=1|X=x) = P(Y=1|Z=0,X=x) P(Z=0|X=x) + P(Y=1|Z=1,X=x) P(Z=1|X=x) \leq e(x)$. 
In this situation, irreducibility is again violated.


In the above situations, irreducibility is violated and $\kappa(F|H) > \kappa^*$. Estimating $\kappa(F|H)$ alone leads to bias. In the following, we will re-examine mixture proportion estimation. In particular, we propose a more general sufficient condition than irreducibility, and introduce an estimation strategy that calls an existing MPE method
and reduces or eliminates its asymptotic bias.

\section{A General Identifiability Condition}

Previous MPE works assume irreducibility. We propose a more general sufficient condition for recovering $\kappa^*$.
\begin{theorem}[Identifiability Under Local Supremal Posterior (LSP)]
    Let $A$ be any non-empty measurable subset of $ E_H = \{x: h(x) > 0 \} $ and $ s = \esssup{x \in A}  P(Y=1|X=x) $. Then
    \begin{align*}
        \kappa^* = s \cdot \inf_{S \subseteq A} \frac{F(S)}{H(S)} 
                 = s \cdot \essinf{x \in A} \frac{f(x)}{h(x)}.
    \end{align*}
    \label{thm:identify}
\end{theorem}


This implies that under LSP, $\kappa^*$ is identifiable.
Two special cases are worthy of comment. First, irreducibility holds when $A = E_H$ and $s = 1$, in which case the above theorem recovers the known result that $\kappa^* = \kappa(F|H)$. Second, when $ A = \left\{ x_0 \right\}$ is a singleton, then $s = P(Y=1|X=x_0)$ and $\kappa^* = P(Y=1|X=x_0) \cdot \frac{f(x_0)}{h(x_0)}$, which can also be derived directly from the definition of conditional probability.


The above theorem gives a general sufficient condition for recovering $\kappa^*$, but estimating $\inf_{S \subseteq A} \frac{F(S)}{H(S)}$ is non-trivial: when $A = E_H$, it can be estimated using existing MPE methods \citep{blanchard2016classification}. When $A$ is a proper subset, however, a new approach is needed. We now present a variation of Theorem \ref{thm:identify} that lends itself to a practical estimation strategy without having to devise a completely new method of estimating $\inf_{S \subseteq A} \frac{F(S)}{H(S)}$.

\begin{theorem}[Identifiability Under Tight Posterior Upper Bound]
Consider any non-empty measurable set $ A \subseteq E_H = \{x: h(x) > 0 \} $, and let $ s = \esssup{x \in A}  P(Y=1|X=x)$. Let $\alpha(x)$ be any measurable function satisfying
    \begin{equation}
        \alpha(x) \in \begin{cases}
                    \left[ P(Y=1|X=x), s \right], & x \in A, \\
                    \left[ P(Y=1|X=x), 1\right], & \text{o.w.}
                \end{cases} 
    \label{eq:alpha_x}
    \end{equation}
    Define a new distribution $\widetilde{F}$ in terms of its density
    \begin{equation}
        \begin{split}
            & \widetilde{f}(x) = \frac{1}{c} \cdot \alpha(x) \cdot f(x), \\
            \text{ where } \quad & c = \int \alpha(x) f(x) dx = \expect{\alpha(X)}{X \sim F}.
        \end{split}
        \label{eq:f_tilde}
    \end{equation} 
    Then 
    \begin{align*}
        \kappa^* = c \cdot \kappa(\widetilde{F}|H).
    \end{align*}
    \label{thm:subsample_theory}
\end{theorem}

\vskip -0.3in
The theorem can be re-stated as: $\kappa^*$ is identifiable given an upper bound of the posterior probability $\alpha(x) \geq P(Y=1|X=x)$ that is tight for some $x \in A$. 
One possible choice for $\alpha(x)$ is simply 
\begin{equation*}
    \alpha(x) = \begin{cases}
        s, & x \in A, \\
        1, & \text{o.w.}
    \end{cases}
\end{equation*}
If the conditional probability $P(Y=1|X=x)$ is known for all $x \in A$, then 
\begin{equation*}
    \alpha(x) = \begin{cases}
        P(Y=1|X=x), & x \in A, \\
        1, & \text{o.w.},
    \end{cases}
\end{equation*}
may be chosen. 

Having $\alpha(x)$ satisfying Eqn. \eqref{eq:alpha_x} ensures identifiablility of $\kappa^*$. Relaxing this requirement slightly still guarantees that the bias will not increase.

\begin{corollary}
    Let $\alpha(x)$ be any measurable function with
    \begin{equation}
         \alpha(x) \in [P(Y=1|X=x), 1] \quad \forall x.
    \label{eq:alpha_loose}
    \end{equation}
    Define a new distribution $\widetilde{F}$ in terms of its density $\widetilde{f}$ according to Eqn. \eqref{eq:f_tilde}.
    Then 
    \begin{align*}
        \kappa^* \leq c \cdot \kappa(\widetilde{F}|H) \leq \kappa(F|H).
    \end{align*}
    \label{corol:upp_bd}
\end{corollary}
\vskip -0.3in
This shows that even if we have a non-tight upper bound on $P(Y=1|X=x)$, the quantity $c \cdot \kappa(\widetilde{F}|H)$ is still bounded by $\kappa(F|H)$. Therefore, a smaller asymptotic bias may be achieved by estimating $c \cdot \kappa(\widetilde{F}|H)$ instead of $\kappa(F|H)$.

The intuition underlying the above results is that the new distribution $\widetilde{F}$ is generated by throwing away some probability mass from $G$, and therefore can be viewed as a mixture of $H$ and a new $\widetilde{G}$, but now $\widetilde{G}$ tends to be irreducible w.r.t. $H$. The proportion $\kappa(\widetilde{F}|H)$ relates to the original proportion $\kappa^*$ by a scaling constant $c$. This interpretation is supported mathematically in Appendix \ref{app:sub_better}.

\section{Subsampling MPE (SuMPE)}

Theorem \ref{thm:subsample_theory} directly motivates a practical algorithm. We obtain a new distribution $\widetilde{F}$ from $F$ by rejection sampling \citep{mackay2003information}, which is a Monte Carlo method that generates a sample following a new distribution $\widetilde{Q}$ based on a sample from distribution $Q$, in terms of their densities $\widetilde{q}$ and $q$. An instance $x$ drawn from $q(x)$ is kept with acceptance probability $\beta(x) \in [0, 1]$, and rejected otherwise. Appendix \ref{app:rej_samp} shows the detailed procedure. In our scenario, $\widetilde{Q} = \widetilde{F}$, $Q = F$ and $\beta(x) = \alpha(x)$.

\subsection{Method}
Our Subsampling MPE algorithm, SuMPE (Algorithm \ref{alg:submpe}), follows directly from Theorem \ref{thm:subsample_theory}. It first obtains in line 3 a data sample $X_{\widetilde{F}}$ following distribution $\widetilde{F}$ using rejection sampling and in line 4 estimates the normalizing constant $c$. Then in line 5, it computes an estimate $\widehat{\kappa}(\widetilde{F}|H)$ using any existing MPE method that consistently estimates the mixture proportion under irreducibility. 
The final estimate is returned as the product of $\widehat{c}$ and $\widehat{\kappa}(\widetilde{F}|H)$. 

Rejection sampling in high dimensional settings may be inefficient due to a potentially low acceptance rate \citep{mackay2003information}. However, this concern is mitigated in our setting  because the acceptance rate can be taken to be 1 except on the set $A$, which is potentially a small set.

\begin{algorithm}[tb]
    \caption{Subsampling MPE (SuMPE)}
    \label{alg:submpe}
 \begin{algorithmic}[1]
    \STATE {\bfseries Input:} \\ \quad $X_F$: sample drawn i.i.d. from $F$ \\
        \quad $X_H$: sample drawn i.i.d. from $H$ \\
        \quad $\alpha(x)$: acceptance function
    \STATE {\bfseries Output:} \\ \quad Estimate of $\kappa^*$
    \STATE Generate $X_{\widetilde{F}}$ from $X_F$ by rejection sampling (Algorithm \ref{alg:rejSample}), with acceptance function $ \alpha(x)$. 
    \STATE Compute $ \widehat{c} = \abs{X_{\widetilde{F}}}/\abs{X_F}$, where $\abs{\cdot}$ denotes the cardinality of a set.
    \STATE Apply an off-the-shelf MPE algorithm to produce an estimate $\widehat{\kappa}(\widetilde{F}|H)$ from $X_{\widetilde{F}}$ and $X_H$.
    \STATE {\bfseries return} $ \widehat{c} \cdot \widehat{\kappa}(\widetilde{F}|H) $
 \end{algorithmic}

\end{algorithm}
One advantage of building our method around existing MPE methods is that we may adapt known theoretical results to our setting. To illustrate this, we give a rate of convergence result for SuMPE.
\begin{theorem} 
Assume $\alpha(x)$ satisfies the condition in Eqn. \eqref{eq:alpha_x}.  After subsampling, assume the resulting $\widetilde{F}$ and $H$ are such that $\argmin{S \in \mathcal{A}: H(S) > 0 } \frac{\widetilde{F}(S)}{H(S)}$ exists. Then there exists a constant $C > 0$ and an existing MPE estimator $\widehat{\kappa}$ such that, for $m$ and $n$ sufficiently large, the estimator $\widehat{\kappa}^*$ obtained from SuMPE (Algorithm \ref{alg:submpe}) satisfies
    \begin{multline*}
        \Pr \left( \abs{\widehat{\kappa}^* - \kappa^*} \leq C \left[ \sqrt{\frac{\log m}{m}} + \sqrt{\frac{\log n}{n}} \right] \right) \\
        \geq 1 -  \mathcal{O} \left( \frac{1}{m} + \frac{1}{n} \right).
    \end{multline*}
    \label{thm:roc}
\end{theorem}

\subsection{Practical Scenarios}
\label{subsec:prac}

Our new sufficient condition assumes knowledge of some set $ A \subseteq E_H $ and $ s = \esssup{x \in A} P(Y=1|X=x) $. However, practically speaking, our algorithm only requires an $\alpha(x)$ satisfying Eqn. \eqref{eq:alpha_x} for some $A \subseteq E_H$ and the associated value of $s$, and does not require the explicit knowledge of $A$ and $s$. Additionally, even if $\alpha(x)$ does not satisfy Eqn. \eqref{eq:alpha_x} , as long as it satisfies Eqn. \eqref{eq:alpha_loose} (which is easier to achieve), it shall perform no worse than directly applying off-the-shelf MPE methods. 

There are settings where a generic construction of $\alpha(x)$ is possible. For example, suppose the user has access to fully labeled data (where it is known which of $G$ or $H$ each instance came from) but only on a subset $A$ of the domain. This may come from an annotator who is only an expert on a subset of instances. This data should be sufficient to get a non-trivial upper bound on the posterior class probability $P(Y|X)$, which in turns leads to an $\alpha(x)$. 

More typically, however, it may be necessary to determine $\alpha(x)$ on a case by case basis. This section continues the discussion of  the three applications introduced in Sec. \ref{subsec:vio}. Each of these three settings leverages different domain-specific knowledge in different ways, and we believe this leads to the best $\alpha(x)$ compared to a one-size-fits-all construction.



\subsubsection{Unfolding}
\label{sec:unfolding}
Unfolding refers to the process of recovering one or more true distributions from contaminated ones \citep{cowan1998statistical}. In gamma spectrum unfolding \citep{li2019review}, a gamma ray detector measures the energies of incoming gamma rays. The gamma rays were emitted either by a source of interest or from the background. The measurement is represented as a histogram $f(x)$ where the bins correspond to a quantization of energy. In many settings, the histogram $h(x)$ of measurements from the source of interest is also known. In this case, unfolding amounts to the estimation of the unknown background histogram $g(x)$. Toward this goal, it suffices to estimate the proportion of recorded particles $\kappa^*$ emanating from the source of interest, since $g(x) = (f(x) - \kappa^* h(x))/(1-\kappa^*)$. This application corresponds to the ``ubiquitious background" setting described in Sec. \ref{subsec:vio}, where irreducibility may not hold since the source of interest energies can be a subset of the background energies.

Using existing techniques from the nuclear detection literature \citep{knoll2010radiation, alamaniotis2013kernel}, we can obtain a lower bound $\rho(x)$ of the quantity $(1-\kappa^*)g(x)$ on a certain set $A \subset \supp(H)$ (see Appendix \ref{app:unfolding_detail} for details). 
This leads to the acceptance function
\begin{equation}
    \alpha(x) = 
        \begin{cases}
            1 - \frac{\rho(x)}{f(x)},  & x \in A, \\
            1, & \text{o.w.},
        \end{cases}
    \label{eq:rejFunc_unfolding}
\end{equation}
which is an upper bound of $P(Y=1|X=x)$, satisfying the condition in Corollary \ref{corol:upp_bd}.

\subsubsection{CPE under domain adaptation}
\label{sec:cpe_da}






In the problem of domain adaptation, the learner is given labeled examples from a source distribution, and the task is to do inference on a potentially different target distribution.
Previous work on domain adaptation mainly focuses on classification and typically makes assumptions about which  of the four distributions $P_X, P_{Y|X}, P_{Y}$, and $P_{X|Y}$ vary between the source and target. This leads to situations such as covariate shift (where $P_X$ changes) \citep{heckman1979sample}, posterior drift (where $P_{Y|X}$ changes) \citep{scott2019generalized}, prior/target shift (where $P_Y$ changes) \citep{storkey2009training}, and conditional shift (where $P_{X|Y}$ changes) \citep{zhang2013domain}. It is also quite commonly assumed that the support of source distribution contains the support of target \citep{heckman1979sample, bickel2009discriminative, gretton2009covariate, storkey2009training, zhang2013domain, scott2019generalized}.


We study class proportion estimation (CPE) under domain adaptation. Prior work on this topic has considered distributional assumptions like those described above \citep{saerens2002adjusting, sanderson2014class, gonzalez2017review}. In this work, we consider the setting where, in addition to labeled examples from the source, the learner has access to labeled positive and unlabeled data from the target. We propose a model that includes covariate shift and posterior drift as special cases. We use  $P^{sr}_{XY}$ and $P^{tg}_{XY}$ to denote source and target distributions. In MPE notation, $F = P^{tg}_X$, $G = P^{tg}_{X|Y=0}$ and $H = P^{tg}_{X|Y=1}$.



\begin{definition}[CSPL] We say that \emph{covariate shift with posterior lift} occurs whenever
    \begin{align*}
        & \forall x \in \supp(P^{sr}_X) \bigcap \supp(P^{tg}_X), \\
        & P^{sr}(Y=1|X=x) \geq P^{tg}(Y=1|X=x),
    \end{align*}
    and ``$=$'' holds for some $x \in \supp(P^{sr}_X) \bigcap \supp(P^{tg}_{X|Y=1}) $.
\end{definition}
Covariate shift is a special case of CSPL when equality always holds.
One motivation for posterior lift is to model labels produced by an annotator who is biased toward one class. It is a type of posterior drift model wherein the posterior changes from source to target \citep{scott2019generalized, cai2021transfer, maity2021linear}. 
Also notice that CSPL does not require the support of the source distribution to contain the target, nor irreducibility.


CSPL is motivated by a marketing application
mentioned in Sec. \ref{subsec:vio}. 
In marketing, companies often have access to labeled data from a source distribution, such as survey results where customers express their interest in a product. Additionally, they also have access to labeled positive and unlabeled data from the target distribution, which corresponds to actual purchasing behavior. In this scenario, the CSPL assumption is often met as it is more likely for customers to express interest than to actually make a purchase: $P^{sr}(Y=1|X=x) \geq P^{tg}(Y=1|X=x)$.

Although irreducibility is violated in the marketing application due to the ``global uncertainty'' about a target customer buying the product (see Sec. \ref{subsec:vio}),
CSPL ensures the identifiability of $\kappa^* = P^{tg}(Y=1)$ because we can choose the set $A = \supp(P^{sr}_X) \bigcap \supp(P^{tg}_{X|Y=1})$ and the acceptance function as
\begin{equation}
    \alpha(x) = \begin{cases}
        P^{sr}(Y=1|X=x), & x \in A, \\
        1, & \text{o.w.},
    \end{cases}
    \label{eq:rejFunc_cs}
\end{equation}
which satisfies the identifiability criteria in Theorem \ref{thm:subsample_theory}. 
\footnote{ 
To see this, take $A' = \{ x \in A: P^{sr}(Y=1|X=x) = P^{tg}(Y=1|X=x) \}$ and $s' = \esssup{x \in A'} P^{sr} (Y=1|X=x) = \esssup{x \in A'} P^{tg} (Y=1|X=x) $. Then $A'$ and $s'$ are the $A$ and $s$ in Theorem \ref{thm:subsample_theory}.
}
By using the labeled source data, an estimate of $P^{sr}(Y=1|X=x)$ can be obtained and used as the acceptance function $\alpha(x)$ in Algorithm \ref{alg:submpe} to do CPE.

\subsubsection{Selected/Reported at Random}
\label{sec:rar}


In public health, $(X,Y,Z)$ are jointly distributed, where $X$ is the feature vector, $Y \in \{0, 1\}$ denotes whether a person reports a medical condition or not and $Z \in \{0, 1\}$ indicates whether a person truly has the medical condition. The goal is to estimate the proportion of people in $X_F$ that report the medical condition. This setting was described in Sec. \ref{subsec:vio} as ``underreported outcomes'' where it was argued that irreducibility may not hold, in which case estimating $\kappa(F|H)$ overestimates the true value of $\kappa^*$. Our SuMPE framework provides a way to eliminate the bias.

The behavior of underreporting can be captured using the \emph{selection bias} model \citep{kato2018learning, bekker2019beyond, gong2021instance}.  Denote the probability of reporting as $e(x) := P(Y=1|X=x, Z=1)$. Assume there is no false report: $\forall x, P(Y=1|X=x, Z=0) = 0$. We use the notation $p(x|\Omega)$ to indicate the conditional density of $X$ given the event $\Omega$. Then $p(x|Y=1) = \frac{e(x)}{\nu} p(x|Z=1)$, where $\nu = P(Y=1|Z=1)$.
Under this model, the density of marginal distribution $P_X$ can be decomposed as
\begin{align*}
    p(x)  = \ & (1-\alpha) \cdot p(x|Z=0) + \alpha \cdot p(x|Z=1)  \\
          = \ & (1-\alpha \nu) \cdot p(x|Y=0) + \alpha \nu \cdot p(x|Y=1),  
\end{align*}    
where $\alpha = P(Z=1)$ is the proportion of people having the medical condition. The mixture proportion to be estimated is $\kappa^* = P(Y=1) = P(Y=1, Z=1) = P(Z=1) P(Y=1|Z=1)  = \alpha \nu$. 

We assume access to i.i.d. sample from $H = P_{X|Y=1}$, representing the public survey data where people report the presence of the medical condition, and from $F = P_X$, representing the target group. 
Further assume that 
$A \subseteq \{x: P(Z=1|X=x) = 1 \}$. 
This is a subset of patients who are guaranteed to have the condition, which could be obtained based on historical patient data from hospital.
Then the mixture proportion $\kappa^* = \alpha \nu$ can be recovered from Algorithm \ref{alg:submpe}, where the acceptance function is 
\begin{equation}
    \alpha(x) = \begin{cases}
        e(x), & x \in A, \\
        1, & \text{o.w.}
    \end{cases}
    \label{eq:rejFunc_underreport}
\end{equation}
$\alpha(x)$ satisfies the condition in Theorem \ref{thm:subsample_theory} 
. This is because under no-false-report assumption, $\forall x \in A, P(Y=1|X=x) = P(Y=1, Z=1|X=x) = P(Y=1|X=x, Z=1) \cdot P(Z=1|X=x) = e(x) \cdot 1 = e(x)$. \footnote{ 
Take $A' = \{x \in A: e(x) > 0 \} $ and $s' = \esssup{x \in A'} e(x) = \esssup{x \in A'} P(Y=1|X=x) $. Then $A'$ and $s'$ are the $A$ and $s$ in Theorem \ref{thm:subsample_theory}.
}
In practice, $e(x)$ can be estimated from labeled examples $(X, Y, Z=1)$.



\section{Limitation of Previous Work}
\label{sec:rempe}

Previous research by \citet{yao2022rethinking} introduced the Regrouping MPE (ReMPE) method \footnote{\citet{yao2022rethinking} called the method Regrouping CPE (ReCPE).}, which is built on top of any existing MPE estimator (just like our meta-algorithm).
They claimed that ReMPE works as well as the base MPE method when irreducibility holds, while improving the performance when it does not. 
In this section we offer some comments on ReMPE.

\paragraph*{Regrouping MPE in theory.}
Consider any $F, H, G, \kappa^*$ such that Eqn. \eqref{eq:mix} holds.
Write $G$ as an arbitrary mixture of two distributions $G =\gamma G_1 + (1-\gamma) G_2, \gamma \geq 0$. 
Then $F$ can be re-written as 
\begin{equation}
    \begin{split}
        F & = (1-\kappa^*) G + \kappa^* H \\
      & = (1-\kappa^*) \left[ \gamma G_1 + (1-\gamma) G_2 \right] + \kappa^* H \\
      & = (1-\kappa^*)(1-\gamma) G_2 + \underbrace{\left[ \textcolor{blue}{(1-\kappa^*)\gamma G_1} + \kappa^* H  \right]}_{\text{Regrouped}} \\    
      & = (1-\kappa') G_2 + \kappa' H',
    \end{split}
    \label{eq:regroup}
\end{equation}
where $\kappa' = \kappa^* + (1-\kappa^*)\gamma$. \citet{yao2022rethinking} assumes there exists a set such that the infimum in Eqn. \eqref{eq:inf} and \eqref{eq:reducible} can be achieved. They
proposed to specify $G_1$ as the truncated distribution of $G$ in set $B$, denote as $G_B$, where $B = \arg \min_{S \in \mathfrak{S}} \frac{G(S)}{H(S)}$. This specific choice causes the resulting distribution $G_2$ to be irreducible w.r.t. $H'$ and the bias introduced by regrouping $(1-\kappa^*)G(A)$ to be minimal. Denote the above procedure as \emph{ReMPE-1} (Algorithm \ref{alg:ReMPE-1}). Theorem 2 in \citet{yao2022rethinking} provides a theoretical justification for ReMPE-1, which we will restate here. 
\begin{algorithm}[tb]
    \caption{ReMPE-1 \citep{yao2022rethinking}}
    \label{alg:ReMPE-1}
 \begin{algorithmic}[1]
    \STATE {\bfseries Input:} Distributions $F$ and $H$
    \STATE Obtain set  $B = \arg \min_{S \in \mathfrak{S}} \frac{G(S)}{H(S)}$
    \STATE Generate new distribution $H'$ by Eqn. \eqref{eq:regroup}, where $G_1 = G_B$
    \STATE {\bfseries return} $\kappa(F|H')$
 \end{algorithmic}
\end{algorithm}
\begin{theorem}[\citet{yao2022rethinking}]
    Let $\kappa(F|H')$ be the mixture proportion obtained from ReMPE-1.
    1) If $G$ is irreducible w.r.t. $H$, then $\kappa(F|H') = \kappa^*$. 
    2) if $G$ is not irreducible w.r.t. $H$, then $\kappa^* < \kappa(F|H') < \kappa(F|H)$.
    \label{thm:regroup}
\end{theorem}

While this theorem is valid, we note that in the case where $G$ is irreducible w.r.t. $H$, the set $B$ is outside the support of $G$, and therefore it is not appropriate to describe the procedure as ``regrouping $G$." In fact, performing regrouping 
($\gamma > 0$) always introduces a positive bias, because $\kappa(F|H') \geq \kappa' > \kappa^*$. This indicates that any kind of regrouping will have a positive bias under irreducibility.

\paragraph*{Regrouping MPE in practice.} \citet{yao2022rethinking}'s practical implementation of regrouping deviates from the theoretical proposal. Here, we state and analyze the idealized version of their practical algorithm, referred to as \emph{ReMPE-2} (Algorithm \ref{alg:ReMPE-2}). 
ReMPE-2 does not rely on the knowledge of $\kappa^*$ and $G(B)$ as outlined in Eqn. \eqref{eq:regroup}. Instead, the set $B$ is chosen based solely on $F$ and $H$, and the distribution $H'$ is obtained through regrouping some probability mass from $F$ rather than $G$. 
\footnote{\citet{yao2022rethinking}'s real implementation differs a bit from ReMPE-2 in that instead of choosing $\arg \min_{S \in \mathfrak{S}} \frac{F(S)}{H(S)}$, they select $p= 10 \%$ of examples drawn from $F$ with smallest estimated score of $f(x)/h(x)$.}


\begin{algorithm}[tb]
    \caption{ReMPE-2 \citep{yao2022rethinking}}
    \label{alg:ReMPE-2}
 \begin{algorithmic}[1]
    \STATE {\bfseries Input:} Distributions $F$ and $H$
    \STATE Obtain set  $B = \arg \min_{S \in \mathfrak{S}} \frac{F(S)}{H(S)}$
    \STATE Generate new distribution $H' = \frac{1}{1 + F(B)} \left( F_B + H  \right)$ 
    \STATE {\bfseries return} $\kappa(F|H')$
 \end{algorithmic}
\end{algorithm}

ReMPE-2 is fundamentally different from ReMPE-1 in that it uses a different way to construct $H'$. To be specific, when the irreducibility assumption holds, ReMPE-1 suggests regrouping nothing (because $G(B) = 0$), but ReMPE-2 still regroups a proportion $F(B)$ from $F$ to $H$. Therefore, Theorem \ref{thm:regroup} does not apply to it. \citet{yao2022rethinking} did not analyze ReMPE-2, but the next result shows that it has a negative bias under irreducibility.
\begin{proposition}
    For $\kappa(F|H')$ obtained from ReMPE-2:
    \begin{align*}
        \kappa(F|H') < \kappa(F|H).
    \end{align*}
    \label{prop:re_under}
\end{proposition}
\vskip -0.3in
Thus, if irreducibility holds, then ReMPE-2 returns $\kappa(F|H') < \kappa(F|H) = \kappa^*$, which is undesirable.
However, when irreducibility does not hold, ReMPE-2 may lead to a smaller asymptotic bias than estimating $\kappa(F|H)$, which could explain why the authors observe empirical improvements in their results. Our theoretical analysis of ReMPE-2 is supported experimentally in Sec. \ref{sec:experiment} and Appendix \ref{app:exp_irred}.

To summarize, \citet{yao2022rethinking} proposed a regrouping approach that was the first attempt to tackle the problem of MPE beyond irreducibility and motivated our work. ReMPE-1 recovers $\kappa^*$ when irreducibility holds (although in this case it is not doing regrouping), and decreases bias when irreducibility does not hold. The more practical algorithm ReMPE-2 might decrease the bias
when irreducibility does not hold, but it has a negative bias when irreducibility does hold.
Like ReMPE-1, SuMPE draws on some additional information beyond $F$ and $H$. Both meta-algorithms do not increase the bias, and recover $\kappa^*$ when irreducibility holds. Unlike ReMPE-1, however, SuMPE is able to recover $\kappa^*$ under a more general condition. Furthermore, our practical implementations of subsampling are based directly on Theorem \ref{thm:subsample_theory}, unlike ReMPE-2 which does not have the desireable theoretical properties of ReMPE-1. Finally, as we argue in the next section, SuMPE offers significant empirical performance gains.

One limitation of our SuMPE framework is that some knowledge of $P(Y|X)$ is needed and that  $\alpha(x)$ may need to be developed specifically for different applications. 




\section{Experiments}
\label{sec:experiment}
We ran our algorithm on nuclear, synthetic and some benchmark datasets taken from the UCI machine learning repository and MNIST, corresponding to all three scenarios described in Sec. \ref{subsec:prac}. 
We take four MPE algorithms: DPL \citep{ivanov2020dedpul}, EN \citep{elkan2008learning}, KM \citep{ramaswamy2016mixture} and TIcE \citep{bekker2018estimating}.
We compare the original version of these methods together with their regrouping (Re-$(\cdot$)) and subsampling (Su-$(\cdot$)) version. All experiments in this section consider settings where irreducibility is violated. The summarized results are shown in Table \ref{table:summary_err} and the detailed results are in Appendix \ref{sec:more_experiment}. Overall, the subsampling version of each MPE method outperforms the original and regrouping version.
Additional experiments where irreducibility holds are offered in Appendix \ref{app:exp_irred}, 
where we find that ReMPE harms the estimation performance while SuMPE does not. 
The implementation is available at \url{https://github.com/allan-z/SuMPE}.

\begin{table*}[h!]

    \caption{ Summarized table of 
    average absolute estimation error, corresponding to testing cases in Sec. \ref{sec:experiment}. 
    Several state-of-the-art MPE algorithms DPL, EN, KM and TIcE are selected. The mean absolute error ($\avg[\abs{\widehat{\kappa}^* - \kappa^*}]$) is reported, the smallest error among original, regrouping and subsampling version is bolded. $(+/-)$ denotes that on average the estimator produces positive$/$negative estimation bias ($\sgn[\avg(\widehat{\kappa}^* - \kappa^*)]$). 
    }
    \label{table:summary_err}
    \vskip 0.15in
    \centering
    \resizebox{\textwidth}{!}{
    \begin{tabular}{ | c | c | c c c | c c c|c c c| c c c | }
    \hline
     Setup & Dataset & DPL & ReDPL & SuDPL & EN & ReEN & SuEN & KM & ReKM & SuKM & TIcE & ReTIcE & SuTIcE \\ \hline
     Unfolding & Gamma Ray & $0.045+$ & $0.117-$ & $\bm{0.013+}$ & $0.034+$ & $0.118-$ & $\bm{0.027-}$ & $0.163+$ & $0.076+$ & $\bm{0.042+}$ & $0.095+$ & $0.061+$ & $\bm{0.019+}$ \\ \hline 
     \multirow{5}{4em}{Domain Adaptation} & Synthetic & $0.060+$ & $0.053+$ & $\bm{0.028-}$ & $\bm{0.045-}$ & $0.061-$ & $0.067-$ & $0.063+$ & $0.059+$ & $\bm{0.022-}$ & $0.128+$ & $0.094+$ & $\bm{0.041+}$ \\ 
      & Mushroom & $0.047+$ & $0.081-$ & $\bm{0.033+}$ & $0.122+$ & $\bm{0.078+}$ & $0.101+$ & $0.067-$ & $0.134-$ & $\bm{0.059-}$ & $0.060+$ & $0.078-$ & $\bm{0.036+}$\\ 
      & Landsat & $0.046+$ & $0.042-$ & $\bm{0.017-}$ & $0.141+$ & $0.110+$ & $\bm{0.085+}$ & $0.046+$ & $0.029+$ & $\bm{0.016-}$ & $0.043+$ & $0.044-$ & $\bm{0.035-}$ \\ 
      & Shuttle & $0.037+$ & $0.138-$ & $\bm{0.015-}$ & $0.090+$ & $0.071+$ & $\bm{0.046+}$ & $0.036+$ & $0.111-$ & $\bm{0.032-}$ & $0.080+$ & $0.086-$ & $\bm{0.038+}$ \\ 
      & MNIST17 & $0.047+$ & $0.085-$ & $\bm{0.028-}$ & $0.231+$ & $0.175+$ & $\bm{0.166+}$ & $0.041+$ & $0.063-$ & $\bm{0.017-}$ & $0.090+$ & $0.073-$ & $\bm{0.048+}$ \\ \hline
     \multirow{4}{4.5em}{Selected/ \\ Reported at Random} & Mushroom & $0.047+$ & $0.095-$ & $\bm{0.027+}$ & $0.119+$ & $0.075+$ & $\bm{0.074+}$ & $0.072-$ & $0.134-$ & $\bm{0.064-}$ & $0.047+$ & $0.066-$ & $\bm{0.044-}$\\ 
      & Landsat & $0.048+$ & $0.057-$ & $\bm{0.019-}$ & $0.142+$ & $0.108+$ & $\bm{0.092+}$ & $0.046+$ & $0.033+$ & $\bm{0.018+}$ & $0.046+$ & $0.053-$ & $\bm{0.041-}$ \\ 
      & Shuttle & $0.035+$ & $0.144-$ & $\bm{0.013-}$ & $0.095+$ & $0.073+$ & $\bm{0.056+}$ & $\bm{0.042+}$ & $0.129-$ & $0.044-$ & $0.079+$ & $0.089-$ & $\bm{0.050+}$ \\ 
      & MNIST17 & $0.047+$ & $0.088-$ & $\bm{0.027+}$ & $0.240+$ & $0.173+$ & $\bm{0.164+}$ & $0.038+$ & $0.064-$ & $\bm{0.022+}$ & $0.096+$ & $0.073+$ & $\bm{0.057+}$ \\ \hline
    \end{tabular}
    }
    
    \vskip -0.1in
    
\end{table*}


\subsection{Unfolding: Gamma Ray Spectra Data}
The gamma ray spectra data are simulated from the Monte-Carlo N-Particle (MCNP) radiation transport code \citep{werner2018mcnp6}. $H$ refers to the distribution of Cesium-137. $G$ is the background distribution, consisting of terrestrial background and Cobalt-60. The goal is to estimate the proportion of Cesium-137 in the measured spectrum $F$.
\begin{align*}
    H & \sim p(x|\text{Cesium}) \\
    G & \sim  0.8 \cdot p(x|\text{Cobalt}) + 0.2 \cdot p(x|\text{terrestrial background}) \\
    F & = (1-\kappa^*) G + \kappa^* H.
\end{align*}
Sample sizes of $m = n = 50,000$ were chosen, which is a reasonable number of counts for many nuclear detection applications. The true mixture proportion $\kappa^*$ is varied in $\left\{0.1, 0.25, 0.5, 0.75 \right\}$. The random variable $x$ represents quantized energy, which is one-dimensional and is discrete-valued. Therefore, we directly use the histogram as the estimate of the distribution. We choose the acceptance function $\alpha(x)$ according to the methodology developed in Sec. \ref{sec:unfolding} and Appendix \ref{app:unfolding_detail}.


Three of the four baseline MPE methods (DPL, EN, KM) did not work well out-of-the-box in this setting. We therefore re-implemented these methods to explicitly leverage the histogram representation of the probability distributions. This also greatly sped up the KM approach. The results are summarized in Table \ref{table:gamma_err}.


\subsection{Domain Adaptation: Synthetic Data}
Following the setup in Sec. \ref{sec:cpe_da}, we specify the target conditional and marginal distributions $H$, $G$ and $F$  as: 
\begin{align*}
    H & \sim \mathcal{N} (\mu_1 = 0, \sigma_1 = 1) \\
    G & \sim 0.8 \cdot \mathcal{N} (\mu_2 = 3, \sigma_2 = 2) + 0.2 \cdot \mathcal{N} (\mu_3 = 4, \sigma_3 = 1) \\
    F & = (1-\kappa^*) G + \kappa^* H.
\end{align*}
where $G$ is not irreducible w.r.t. $H$ because it contains a Gaussian distribution with a bigger variance than $H$. We draw $m=n=1000$ instances from both $H$ and $F$. The true mixture proportion $\kappa^*$ is varied in $\left\{0.1, 0.25, 0.5, 0.75 \right\}$. 

In addition, we draw  $4000$ labeled instances from the source distribution, where $\mu_3$ is changed to $5$. We then truncate the source distribution (and therefore the source sample) to $(- \infty, 2]$. The resulting source and target distribution satisfy the CSPL assumption. 
A $1$ hidden layer neural network with $16$ neurons was trained to predict $P^{sr}(Y=1|X=x)$ for $ x \in (-\infty, 2]$, thus $A = (-\infty, 2]$ and $\alpha(x)$ was chosen according to Eqn. \eqref{eq:rejFunc_cs}. This procedure was repeated $10$ times with different random seeds. Detailed results are shown in Table \ref{table:da_syn_err}.

\subsection{Domain Adaptation: Benchmark Data}
In many machine learning datasets for classification (e.g., those in UCI), irreducibility is satisfied. 
\footnote{The datasets Mushroom, Landsat, Shuttle and MNIST were chosen for our study because previous empirical research \citep{ivanov2020dedpul} showed that the baseline MPE methods perform well on these datasets when the irreducibility assumption holds. Our paper focuses on how to eliminate the estimation bias that arises from the violation of irreducibility. Therefore, we chose to use datasets where the baseline methods perform well, in order to clearly observe and measure the bias introduced by MPE methods when irreducibility is not met.} 
Here we manually create datasets that violate irreducibility by uniformly sampling out $90 \%$ (or $80 \%$) of the original positive data as the target positive data, with all the remaining data treated as target negative. The target conditional and marginal distributions $H$, $G$ and $F$ are specified as: 
\begin{align*}
    H & \sim p(x|Y=1) \\
    G & \sim \gamma p(x|Y=1) + (1-\gamma) p(x|Y=0) \\
    F & = (1-\kappa^*) G + \kappa^* H.
\end{align*}
We draw $m = 1000$ instances from $H$ and $n \in \left[1000, 4000\right]$ instances from $F$ (the exact number varies by datasets and is based on number of examples available, see the code provided). The true mixture proportion $\kappa^*$ is varied in $\left\{0.1, 0.25, 0.5, 0.75 \right\}$. The proportion $\gamma$ is determined by the total number of positive and negative data originally available in the dataset, therefore varies case by case. 

In addition, we obtain labeled instances following the source distribution, by drawing  $\kappa^* n$ data from the target positive and $0.95 (1-\kappa^*) n$ data from the target negative distribution. This causes a prior shift 
that simulates CSPL.

A $2$ hidden layer neural network with $512$ neurons was trained to predict $P^{sr}(Y=1|X=x)$. For real-world high-dimensional data, it is hard to know the support. Therefore, we choose $A = \{x: \widehat{P}^{sr}(Y=1|X=x) > 0.5 \}$,
because an example $x$ with high $\widehat{P}^{sr}(Y=1|X=x)$ is more likely to lie in the support of $H$.
The acceptance function $\alpha(x)$ was determined according to Eqn. \eqref{eq:rejFunc_cs}. The above procedure was repeated $10$ times with different random seeds. Table \ref{table:da_bench_err} summarizes the results.

\subsection{Selected/Reported at Random: Benchmark Data}
Recalling the setting of Sec. \ref{sec:rar}, there is a jointly distributed triple $(X,Y,Z)$, where $Y$ indicates whether a condition is reported, and $Z$ indicates whether the condition is actually present. For the experiments below, the data from $F, G$, and $H$ are generated in the same way as the target distribution of the previous subsection. Instead of observing labeled source data, however, in this subsection we instead observe $\kappa^* n$ instances from $p(x|Z=1)$, together with their labels $Y$, matching the setup of Sec. \ref{sec:rar}.

A $2$ hidden layer neural network with $512$ neurons was trained to predict $e(x) = P(Y=1|X = x, Z=1)$. 
We chose $A = \{x: \widehat{e}(x) > 0.6 \}$ 
\footnote{
$A$ needs to be a subset of $\supp(H)$. Note that in population level, $h(x) = p(x|Y=1) \propto e(x) \cdot p(x|Z=1) $, thus $ \{x: e(x) > 0.6 \} \subseteq \{x: e(x) > 0 \} = \{x: h(x) > 0 \}$. (The last equality holds because $e(x) = P(Y = 1|X=x, Z=1)$.) Here we replace $e(x)$ with $\widehat{e}(x)$.
}
and $\alpha(x)$ according to Eqn. \eqref{eq:rejFunc_underreport}. The above procedure was repeated $10$ times with different random seeds. The results are shown in Table \ref{table:rar_bench_err}.

\section{Conclusion}
This work introduces a more general identifiability condition than irreducibility for mixture proportion estimation. 
We also propose a subsampling-based framework that achieves bias reduction/elimination for baseline MPE algorithms.
Theoretically, our work expands the scope of settings where MPE can be solved. Practically, we illustrate three scenarios where irreducibility fails, and our meta-algorithm successfully improves upon baseline MPE methods.

\section*{Acknowledgements}
This work was supported in part by the National Science Foundation under award 2008074 and the  Department of Defense, Defense Threat Reduction Agency under award HDTRA1-20-2-0002.




\bibliography{refs}
\bibliographystyle{icml2023}

\newpage
\appendix
\onecolumn

\section{Proofs}
\subsection{Proof of Proposition \ref{prop:post}}
\begin{proposition*}
    Under the latent label model, 
    \begin{equation*}
        \esssup{x} P(Y=1|X=x) = \frac{\kappa^*}{\kappa(F|H)} \ ,
    \end{equation*}
    where $0/0 := 0$.
\end{proposition*}

\begin{proof}
    First consider the trivial case when $\kappa(F|H) = 0$, then from the fact that $\kappa^* \leq \kappa(F|H)$, $\kappa^*$ must also be $0$. According to Eqn. \eqref{eq:cond}, $P(Y=1|X=x) = 0 \ \forall x$, therefore the equality holds.
    
    Then for $\kappa(F|H) > 0$, we consider the cases $\kappa^* > 0$ and $\kappa^*=0$ separately. For the first case, consider two subcases: $h(x) > 0$ and $h(x) = 0$. In the first subcase, we have that $f(x) > 0$, and therefore by the definition of conditional probability,
    \begin{align*}
        P(Y=1|X=x) = \frac{\kappa^*}{\frac{f(x)}{h(x)}}.
    \end{align*}
    Taking the essential supremum over all $x$ with $h(x) > 0$,
    \begin{align*}
        \esssup{x: h(x) > 0} \ P(Y=1|X=x) & = \frac{\kappa^*}{ \essinf{x: h(x) > 0} \frac{f(x)}{h(x)}} \\ 
        & = \frac{\kappa^*}{\kappa(F|H)},
    \end{align*}
    where the second equality follows from Proposition \ref{prop:inf}. In the second subcase, $P(Y=1|X=x)$ is zero. Therefore, 
    \begin{align*}
        \esssup{x}  P(Y=1|X=x) &= \esssup{x: h(x) > 0}  P(Y=1|X=x) \\
        &= \frac{\kappa^*}{\kappa(F|H)}.
    \end{align*}

    When $\kappa^* = 0$, $P(Y=1|X=x) = 0$ for all $x$, and the equality still holds.
\end{proof}

\subsection{Proof of Theorem \ref{thm:identify}}
\begin{theorem*}[Identifiability Under Local Supremal Posterior (LSP)]
    Let $A$ be any non-empty measurable subset of $ E_H = \{x: h(x) > 0 \} $ and $ s = \esssup{x \in A}  P(Y=1|X=x) $, then
    \begin{align*}
        \kappa^* = s \cdot \inf_{S \subseteq A} \frac{F(S)}{H(S)} 
                 = s \cdot \essinf{x \in A} \frac{f(x)}{h(x)}.
    \end{align*}
\end{theorem*}
\begin{proof} Consider the case of $\kappa^* > 0$ and $\kappa^* = 0$ separately. 

    If $ \kappa^* > 0$, then $ E_H = \{x: h(x) > 0 \} \subseteq E_F = \{x: f(x) > 0 \} $. Recall from Eqn. \eqref{eq:cond},
    \begin{align*}
        P(Y=1|X=x) = \kappa^* \cdot \frac{h(x)}{f(x)} \quad \text{when } f(x) > 0.
    \end{align*}
    Taking the essential supremum over $A$ and recall the definition of $s$,
    \begin{align*}
        s = \esssup{x \in A}  P(Y=1|X=x) & = \kappa^* \cdot  \esssup{x \in A} \frac{h(x)}{f(x)}.
    \end{align*}
    Since $A \subseteq E_H \subseteq E_F $, $f(x)$ and $h(x)$ are both positive for $ x \in A$. Rearrange the denominator 
    \begin{align*}
        s = \kappa^* \cdot \frac{1}{\essinf{x \in A} \frac{f(x)}{h(x)}},
    \end{align*}
    take the denominator to the other side, we get
    \begin{align*}
        \kappa^* = s \cdot \essinf{x \in A} \frac{f(x)}{h(x)} = s \cdot \inf_{S \subseteq A} \frac{F(S)}{H(S)},
    \end{align*}
    where the second equality follows from the identity that $\essinf{x \in A} \frac{f(x)}{h(x)} = \inf_{S \subseteq A}  \frac{F(S)}{H(S)}$.

    If $ \kappa^* = 0$, then $s = 0$, the above equality still holds.
\end{proof}

\subsection{Proof of Theorem \ref{thm:subsample_theory}}
\begin{theorem*}[Identifiability Under Tight Posterior Upper Bound]
Consider any non-empty measurable set $ A \subseteq E_H = \{x: h(x) > 0 \} $, and let $ s = \esssup{x \in A}  P(Y=1|X=x)$. Let $\alpha(x)$ be any measurable function satisfying
    \begin{equation}
        \alpha(x) \in \begin{cases}
                    \left[ P(Y=1|X=x), s \right], & x \in A, \\
                    \left[ P(Y=1|X=x), 1\right], & \text{o.w.}
                \end{cases} 
    \label{eq:alpha_x_app}
    \end{equation}
    Define a new distribution $\widetilde{F}$ in terms of its density
    \begin{equation}
        \begin{split}
            & \widetilde{f}(x) = \frac{1}{c} \cdot \alpha(x) \cdot f(x), \\
            \text{ where } \quad & c = \int \alpha(x) f(x) dx = \expect{\alpha(x)}{X \sim F}.
        \end{split}
        \label{eq:f_tilde_app}
    \end{equation} 
    Then 
    \begin{align*}
        \kappa^* = c \cdot \kappa(\widetilde{F}|H).
    \end{align*}
\end{theorem*}


\begin{proof}
    Write $\kappa(\widetilde{F}|H) $ explicitly according to Proposition \ref{prop:inf} and Eqn. \eqref{eq:f_tilde_app}, 
    \begin{align*}
        c \cdot \kappa(\widetilde{F}|H) 
        & = c \cdot \essinf{x: h(x) > 0} \frac{\widetilde{f}(x)}{h(x)} \quad \text{by definition of $\kappa(\widetilde{F}|H)$} \\ 
        & = c \cdot \essinf{x: h(x) > 0} \frac{ \frac{1}{c} \cdot \alpha(x) \cdot f(x) }{h(x)} \quad \text{plug in the expression of $\widetilde{f}(x)$} \\
        & = \essinf{x: h(x) > 0} \frac{\alpha(x) \cdot f(x)}{h(x)} .
    \end{align*}

    We will show that it equals to $\kappa^*$ by proving it is both upper and lower bound of $\kappa^*$.

    From Eqn. \eqref{eq:alpha_x_app}, we can conclude that  $\alpha(x) \geq P(Y=1|X=x) \ \ \forall x$, then
    \begin{align*}
        c \cdot \kappa(\widetilde{F}|H) & \geq \essinf{x: h(x) > 0} \frac{P(Y=1|X=x) \cdot f(x)}{h(x)}  \\
        & = \essinf{x: h(x) > 0} \kappa^* \quad \text{rearrange Eqn. \eqref{eq:cond}} \\
        & = \kappa^* .
    \end{align*}
    Meanwhile, consider $x \in A$
    \begin{align*}
        c \cdot \kappa(\widetilde{F}|H) 
        & = \essinf{x: h(x) > 0} \frac{\alpha(x) \cdot f(x)}{h(x)} \\
        & \leq \essinf{x \in A} \frac{\alpha(x) \cdot f(x)}{h(x)} \quad \text{replace $E_H$ by $A$} \\
        & \leq \essinf{x \in A} \frac{s \cdot f(x)}{h(x)} \quad \text{because $\alpha(x) \leq s, \forall x \in A$} \\
        & = \kappa^* \quad \text{by Theorem \ref{thm:identify}}
    \end{align*}
    This shows that $c \cdot \kappa(\widetilde{F}|H) = \kappa^*$.
\end{proof}

\subsection{Proof of Corollary \ref{corol:upp_bd}}
\begin{corollary*}
    Let $\alpha(x)$ be any measurable function with
    \begin{equation*}
         \alpha(x) \in [P(Y=1|X=x), 1] \quad \forall x.
    \end{equation*}
    Define a new distribution $\widetilde{F}$ in terms of its density $\widetilde{f}$, obtained by Eqn. \eqref{eq:f_tilde}.
    Then 
    \begin{align*}
        \kappa^* \leq c \cdot \kappa(\widetilde{F}|H) \leq \kappa(F|H).
    \end{align*}
\end{corollary*}

\begin{proof}
    From the proof of Theorem \ref{thm:subsample_theory}, we know that
    \begin{align*}
        c \cdot \kappa(\widetilde{F}|H)  = \essinf{x: h(x) > 0} \frac{\alpha(x) \cdot f(x)}{h(x)}
    \end{align*}
    From the fact that $\alpha(x) \geq P(Y=1|X=x) \ \forall x$, we have
    \begin{align*}
        c \cdot \kappa(\widetilde{F}|H) & \geq \essinf{x: h(x) > 0} \frac{P(Y=1|X=x) \cdot f(x)}{h(x)} \\
        & = \essinf{x: h(x) > 0} \kappa^* \\
        & = \kappa^*.
    \end{align*}
    What's more, since $\alpha \leq 1 \ \forall x$,
    \begin{align*}
        c \cdot \kappa(\widetilde{F}|H) \leq \essinf{x: h(x) > 0} \frac{1 \cdot f(x)}{h(x)} = \kappa(F|H) .
    \end{align*}
\end{proof}

\subsection{Proof of Theorem \ref{thm:roc}}

\label{append:rate}
We now establish a rate of convergence result for estimator $\widehat{\kappa}^* = \widehat{c} \cdot \widehat{\kappa}(\widetilde{F}|H)$ from Algorithm \ref{alg:submpe}.
\begin{theorem*} 
Assume $\alpha(x)$ satisfies the condition in Eqn. \eqref{eq:alpha_x}.  After subsampling, assume the resulting $\widetilde{F}$ and $H$ are such that $\argmin{S \in \mathcal{A}: H(S) > 0 } \frac{\widetilde{F}(S)}{H(S)}$ exists. Then there exists a constant $C > 0$ and an existing estimator $\widehat{\kappa}$ such that for $m$ and $n$ sufficiently large, the estimator $\widehat{\kappa}^*$ from Algorithm \ref{alg:submpe} satisfies
    \begin{equation*}
        \Pr \left( \abs{\widehat{\kappa}^* - \kappa^*} \leq C \left[ \sqrt{\frac{\log m}{m}} + \sqrt{\frac{\log n}{n}} \right] \right) \geq 1 -  \mathcal{O} \left( \frac{1}{m} + \frac{1}{n} \right).
    \end{equation*}
\end{theorem*}

\begin{proof}
    Recall the setup, we originally have i.i.d. sample 
    \footnote{
    The notation here is a bit different from Eqn. \eqref{eq:setup} in that the index of $X_{i}$ is changed. This allows for more concise notation in the following derivation.
    }
    \begin{equation*}
        \begin{split}
            X_F & := \left\{ X_{1}, X_{2}, \cdots, X_{n} \right\} \stackrel{iid}{\sim} F, \\
            X_H & := \left\{ X_{n+1}, X_{n+2}, \cdots, X_{n+m} \right\} \stackrel{iid}{\sim} H.  
        \end{split} 
    \end{equation*}


    After rejection sampling, we obtain $n'$ i.i.d. sample $X_{\widetilde{F}} \sim \widetilde{F}$ (where $\expect{n'}{} = c \cdot n$ and $c = \int \alpha(x) f(x) dx$), from which we can estimate $\kappa(\widetilde{F}|H)$.
    Under the assumption that $\argmin{S \in \mathcal{A}: H(S) > 0 } \frac{\widetilde{F}(S)}{H(S)}$ exists, estimator $\widehat{\kappa}(\widetilde{F}|H)$ by \citet{scott2015rate} has rate of convergence
    \begin{equation}
        \Pr \left( \abs{\widehat{\kappa}(\widetilde{F}|H) - \kappa(\widetilde{F}|H)} \leq C_1 \left[ \sqrt{\frac{\log m}{m}} + \sqrt{\frac{\log n'}{n'}} \right] \right) \geq 1 - \frac{2}{m} - \frac{2}{n'}.
        \label{eq:roc2015}
    \end{equation}

    Now $n'$ is a random variable here, and we want to establish a rate of convergence result involving $n$. This can be done by applying a concentration inequality for $n'$. 

    \begin{theorem*}(Hoeffding's Inequality)
        Let $Z_1, \cdots Z_n$ be independent random variables on $\mathbb{R}$ that take values in $[0, 1]$. Denote $Z = \frac{1}{n} \sum_{i=1}^n Z_i$, then for all $\epsilon > 0$,
        \begin{equation*}
            \Pr \left(\abs{\frac{1}{n} \sum_{i=1}^n Z_i - \expect{Z}{}} \leq \epsilon  \right) \geq 1 - 2 \exp(-2n\epsilon^2).
        \end{equation*}
        Let $\delta = 2 \exp{(-2n\epsilon^2)}$, the theorem can be restated as:
        \begin{equation*}
            \Pr \left(\abs{\frac{1}{n} \sum_{i=1}^n Z_i - \expect{Z}{}} \leq \sqrt{ \frac{\log (1/\delta)}{2n} }  \right) \geq 1 - \delta.
        \end{equation*}
    \end{theorem*}

    Take $Z_i = \indicator{\{V_i \leq \alpha(X_{i}) \} } $, where $i \in \{1, 2, \cdots, n\}$, $V_i$ denotes the $i$-th independent draw from $\text{Uniform}(0,1)$ 
    \footnote{$V_i$ is used in rejection sampling (Algorihm \ref{alg:rejSample}).} 
    and $X_{i}$ denote the $i$-th draw from $F$. Then
    \begin{align*}
        & n' = \sum_{i=1}^n Z_i, \\
        & \expect{Z}{} = \frac{1}{n} \sum_{i=1}^n \expect{Z_i}{} = \expect{Z_i}{(X_i, V_i)} = \expect{ \expect{\indicator{\{V_i \leq \alpha(X_i) \} }}{V_i|X_i} }{X_i} = \expect{ \alpha(X_i) }{X_i} 
        = \int \alpha(x) f(x) dx = c.
    \end{align*}
    Plug into Hoeffding's Inequality and setting $\epsilon = \frac{c}{2}$, we have 
    \begin{equation}
        \Pr \left(\abs{n' - c \cdot n} \leq \frac{c}{2} \cdot n  \right) \geq 1 - 2 \exp \left(- \frac{c}{2} n \right),
        \label{eq:bound_n}
    \end{equation}
    which allows us to bound $n'$ by a constant times $n$.

    Now we can establish a rate of convergence result of $\widehat{\kappa}(\widetilde{F}|H)$ w.r.t. $n$
    \begin{align*}
        & \Pr \left( \abs{\widehat{\kappa}(\widetilde{F}|H) - \kappa(\widetilde{F}|H)} \leq C_2 \left[ \sqrt{\frac{\log m}{m}} + \sqrt{\frac{\log n}{n}} \right] \right) \\
        \geq \ & \Pr \left(
        \left(\abs{\widehat{\kappa}(\widetilde{F}|H) - \kappa(\widetilde{F}|H)} \leq C_2 \left[ \sqrt{\frac{\log m}{m}} + \sqrt{\frac{\log n}{n}} \right] \right) 
        \text{and} 
        \left( \frac{c}{2} \cdot n \leq n' \leq \frac{3c}{2} \cdot n  \right)
        \right) \quad \because \Pr (A) \geq \Pr (A \cap B) \\
        = \ & \Pr \left(
        \left(\abs{\widehat{\kappa}(\widetilde{F}|H) - \kappa(\widetilde{F}|H)} \leq C_2 \left[ \sqrt{\frac{\log m}{m}} + \sqrt{\frac{\log n}{n}} \right] \right) 
        \Biggl|
        \left( \frac{c}{2} \cdot n \leq n' \leq \frac{3c}{2} \cdot n  \right)
        \right) \cdot \Pr \left( \frac{c}{2} \cdot n \leq n' \leq \frac{3c}{2} \cdot n  \right) \\
        \geq \ & \left( 1 - \mathcal{O} \left( \frac{1}{m} + \frac{1}{n} \right) \right) \cdot \left( 1 - 2 \exp \left(- \frac{c}{2} n \right) \right) \quad \because  \text{$n$ and $n'$ can be used interchangeably in the first probability term} \\
        \geq \ & 1 - \mathcal{O} \left( \frac{1}{m} + \frac{1}{n} \right) - 2 \exp \left(- \frac{c}{2} n \right) \quad \because (1-a)(1-b) \geq 1 - a - b \\
        \geq \ & 1 - \mathcal{O} \left( \frac{1}{m} + \frac{1}{n} \right) \quad \because \exp (-n) \text{ decays faster}
    \end{align*}
    
    
    As for the estimator of $c = \expect{\alpha(x)}{X \sim F}$, the rate of convergence can also be shown via Hoeffding's Inequality,
    \begin{align*}
        \widehat{c} = \frac{n'}{n} = \frac{1}{n} \sum_{i=1}^n Z_i, \quad c = \expect{Z}{}.
    \end{align*}
    
    Plug into Hoeffding's Inequality and let the confidence $\delta = 1/n$, we have 
    \begin{equation}
        \Pr \left(\abs{\widehat{c} - c} \leq \sqrt{ \frac{\log n}{2n} }  \right) \geq 1 - 1/n.
    \end{equation}
    
    

    
    Note that by triangle inequality
    \begin{align*}
        \abs{ \widehat{c} \cdot \widehat{\kappa}(\widetilde{F}|H) - c \cdot \kappa(\widetilde{F}|H) } & = \abs{ \widehat{c} \cdot \widehat{\kappa}(\widetilde{F}|H) - \widehat{c} \cdot \kappa(\widetilde{F}|H) + \widehat{c} \cdot \kappa(\widetilde{F}|H) - c \cdot \kappa(\widetilde{F}|H) } \\
        & \leq \abs{ \widehat{c} \cdot \widehat{\kappa}(\widetilde{F}|H) - \widehat{c} \cdot \kappa(\widetilde{F}|H) } + \abs{ \widehat{c} \cdot \kappa(\widetilde{F}|H) - c \cdot \kappa(\widetilde{F}|H) } \\
        & = \abs{\widehat{c}} \cdot \abs{ \widehat{\kappa}(\widetilde{F}|H) - \kappa(\widetilde{F}|H) } + \abs{ \kappa(\widetilde{F}|H) } \cdot \abs{ \widehat{c} - c } \\
        & \leq \abs{ \widehat{\kappa}(\widetilde{F}|H) - \kappa(\widetilde{F}|H) } + \abs{ \widehat{c} - c }.
    \end{align*}
    Finally, combine all previous results
    \begin{align*}
         & \Pr \left( \abs{\widehat{\kappa}^* - \kappa^*} \leq C \left[ \sqrt{\frac{\log m}{m}} + \sqrt{\frac{\log n}{n}} \right] \right) \\
         = \ & \Pr \left( \abs{ \widehat{c} \cdot \widehat{\kappa}(\widetilde{F}|H) - c \cdot \kappa(\widetilde{F}|H) } \leq C \left[ \sqrt{\frac{\log m}{m}} + \sqrt{\frac{\log n}{n}} \right] \right) \\
         \geq \ & \Pr \left(  \abs{ \widehat{\kappa}(\widetilde{F}|H) - \kappa(\widetilde{F}|H) } + \abs{ \widehat{c} - c }  \leq C \left[ \sqrt{\frac{\log m}{m}} + \sqrt{\frac{\log n}{n}} \right] \right)  \\
         \geq \ & \Pr \left( \left( \abs{ \widehat{\kappa}(\widetilde{F}|H) - \kappa(\widetilde{F}|H) } \leq \frac{C}{2} \left[ \sqrt{\frac{\log m}{m}} + \sqrt{\frac{\log n}{n}} \right] \right) \text{and} \left( \abs{ \widehat{c} - c } \leq \frac{C}{2} \left[ \sqrt{\frac{\log m}{m}} + \sqrt{\frac{\log n}{n}} \right] \right) \right) \\
         \geq \ & 1 - \mathcal{O} \left( \frac{1}{m} + \frac{1}{n} \right),
    \end{align*}

    then we can conclude that $\widehat{\kappa}^*  \rightarrow \kappa^*$ with rate of convergence $O \left(\sqrt{\frac{\log m}{m}} + \sqrt{\frac{\log n}{n}} \right)$.
    
\end{proof}

\subsection{Proof of Proposition \ref{prop:re_under}}
\begin{proposition*}
    For $\kappa(F|H')$ obtained from ReMPE-2:
    \begin{align*}
        \kappa(F|H') < \kappa(F|H).
    \end{align*}
\end{proposition*}

\begin{proof}
    From the fact that $B = \argmin{S \in \mathfrak{S}} \frac{F(S)}{H(S)}$, we have
    \begin{align*}
        \kappa(F|H) = \inf_{S \in \mathcal{A}: H(S) > 0 } \frac{F(S)}{H(S)} = \frac{F(B)}{H(B)} .
    \end{align*}
    After regrouping, $ H'(B) > H(B) $. Therefore,
    \begin{align*}
        \kappa(F|H') & = \inf_{S \in \mathcal{A}: H'(S) > 0 } \frac{F(S)}{H'(S)} \\
                     & = \frac{F(B)}{H'(B)} \\
                     & < \frac{F(B)}{H(B)} \\
                     & = \kappa(F|H).
    \end{align*}
    
\end{proof}

\section{More about Subsampling MPE}

\subsection{Intuition}
\label{app:sub_better}
Theorem \ref{thm:subsample_theory} and Corollary \ref{corol:upp_bd} have already justified the use of subsampling. Here, we explain in another perspective (in terms of distributions, similar to the analysis in \citet{yao2022rethinking}).
The idea is that, since the original $G$ may violate irreducibility assumption, we modify $G$ such that the resulting latent component distribution is less likely to violate the assumption. 

Write the unknown distribution $G$ itself as a mixture $G = \gamma G_1 + (1-\gamma) G_2, \ \gamma \geq 0$. Then $F$ becomes: 
\begin{align*}
    F & = (1-\kappa^*) G + \kappa^* H \\
      & = (1-\kappa^*) \left[ \gamma G_1 + (1-\gamma) G_2 \right] + \kappa^* H. 
\end{align*}
Switch $G_1$ to the other side (i.e., discarding the probability mass from $G_1$)
\begin{align*}
    F - \textcolor{red}{(1-\kappa^*)\gamma G_1}  & = (1-\kappa^*)(1-\gamma) G_2 + \kappa^* H,
\end{align*}
the left hand side can be re-written as
\begin{align*}
    \underbrace{ \left[ 1-(1-\kappa^*)\gamma \right] }_{:=c, \text{ normalizing const.}} \underbrace{\frac{1}{1-(1-\kappa^*)\gamma} \left[F - \textcolor{red}{(1-\kappa^*)\gamma G_1} \right]  }_{:= \widetilde{F}, \textbf{ after subsampling} }. 
\end{align*}

Then the resulting distribution $\widetilde{F}$ is 
\begin{align*}
    \widetilde{F} & = \frac{(1-\kappa^*)(1-\gamma)}{c} G_2 + \frac{\kappa^*}{c} H \\
                  & =: (1-\tilde{\kappa}^*) \widetilde{G} + \tilde{\kappa}^* H.
\end{align*}

By discarding a portion of probability mass from $G$, which is done by subsampling in practice, the resulting latent component distribution $\widetilde{G}$ is less likely to violate the irreducibility assumption. The new mixture proportion $\tilde{\kappa}^* =  \kappa^*/c$.

In the following, we provide justification of the above claim, which can also be seen as a reformulation of Corollary \ref{corol:upp_bd}.

\begin{proposition}
    Given some probability mass from $G$ being dropped out, $c \cdot \kappa(\widetilde{F}|H)$ is always bounded by: 
    $$ \kappa^* \leq c \cdot \kappa(\widetilde{F}|H) \leq \kappa(F|H). $$
    Furthermore, bias is strictly reduced when $(1-\gamma) \kappa(G_2|H) < \kappa(G|H) $.
    \label{prop:sub_better}
\end{proposition}

\begin{proof}
    Observe that
    \begin{align*}
        c \cdot \kappa(\widetilde{F}|H) & = \inf_{S \in \mathcal{A}: H(S) > 0 } \frac{c \cdot \widetilde{F}(S)}{H(S)} \\
        & = \kappa^* + (1-\kappa^*) \inf_{S \in \mathcal{A}: H(S) > 0 } \frac{(1-\gamma) G_2(S)}{H(S)},
    \end{align*}    
    therefore $c \cdot \kappa(\widetilde{F}|H) \geq \kappa^*$.
    
    From the fact that $G = \gamma G_1 + (1-\gamma) G_2$, we have $G \geq (1-\gamma) G_2$. Then
    \begin{align*}
        c \cdot \kappa(\widetilde{F}|H)
        & \leq \kappa^* + (1-\kappa^*) \inf_{S \in \mathcal{A}: H(S) > 0 } \frac{G(S)}{H(S)}  \\
        & = \kappa(F|H).
    \end{align*}    
    To have bias reduction, we need $c \cdot \kappa(\widetilde{F}|H) < \kappa(F|H)$, where both quantities can be represented as
    \begin{align*}
        c \cdot \kappa(\widetilde{F}|H) & = \kappa^* + (1-\kappa^*)(1-\gamma) \kappa(G_2|H) \\
        \kappa(F|H) & = \kappa^* + (1-\kappa^*) \kappa(G|H).
    \end{align*}
    Then we can get the result of $(1-\gamma) \kappa(G_2|H) < \kappa(G|H) $ by direct comparison.
\end{proof}

Based on the above proof, we claim that $c \cdot \kappa(\widetilde{F}|H)$ leads to no worse estimation bias than $\kappa(F|H)$. To be specific, when $G$ is irreducible w.r.t. $H$, then $\kappa^* =  c \cdot \kappa(\widetilde{F}|H) = \kappa(F|H)$. When $G$ is not irreducible w.r.t. $H$, then $\kappa^* \leq  c \cdot \kappa(\widetilde{F}|H) \leq \kappa(F|H)$. 
The key difference compared to \citet{yao2022rethinking}'s approach is that we are modifying $G$, thus equivalently subsampling $F$, rather than regrouping on $H$.

In summary, without making assumptions about irreducibility, as long as we remove some contribution from $G$ in $F$ by subsampling, the resulting identifiable quantity $c \cdot \kappa(\widetilde{F}|H)$ will be at least no worse than the maximum proportion $\kappa(F|H)$.
Furthermore, with the knowledge of set $A$ and $\esssup{x \in A} P(Y=1|X=x)$, $c \cdot \kappa(\widetilde{F}|H)$ will equal to $\kappa^*$. 


\subsection{Rejection Sampling}
\label{app:rej_samp}
Rejection sampling is a Monte Carlo method that aims to generate sample following a new distribution $\widetilde{Q}$ based on sample from distribution $Q$, which are characterised by the densities $\widetilde{q}$ and $q$. An instance $x$ drawn from $q(x)$ is kept with acceptance probability $\beta(x) \in [0, 1]$, and rejected otherwise. Algorithm \ref{alg:rejSample} shows the detailed procedure.

\begin{algorithm}[H]
   \caption{Rejection Sampling}
   \label{alg:rejSample}
\begin{algorithmic}
   \STATE {\bfseries Input:} \\ \quad sample $[x_i] = \{ x_{1}, \cdots, x_{n} \} \sim q(x)$, \\ 
                \quad acceptance function $\beta(x)$
    \STATE {\bfseries Output:} \\ \quad sample $ \left[ z_{j} \right] \sim \widetilde{q}(x) = \frac{1}{c} \cdot \beta(x) \cdot q(x) $, \\
    \quad \quad where $c = \int \beta(x) q(x) dx = \expect{\beta(X)}{X \sim q} $
   \STATE Initialize $j \gets 1$
   \FOR{$i=1$ {\bfseries to} $n$}
   \STATE $ v \sim \text{Uniform}(0,1) $
   \IF{$v \leq \beta(x)$}
   \STATE $z_j \gets x_{i}$
   \STATE $j \gets j+1$
   \ENDIF
   \ENDFOR
   \STATE {\bfseries return} $ \left[ z_{j} \right] $

\end{algorithmic}
\end{algorithm}



\subsection{Gamma Spectrum Unfolding}
\label{app:unfolding_detail}

In gamma spectrum unfolding, a gamma ray detector measures the energies of incoming gamma ray particles. The gamma rays were emitted either by a source of interest or from the background. The measurement is represented as a histogram $f(x)$, where the bins correspond to a quantization of energy. The histogram $h(x)$ of measurements from the source of interest is also known.

The goal is to obtain a lower bound $\rho(x)$ of the quantity $(1-\kappa^*)g(x)$ on a certain set $A \subset \supp(H)$. We specify $A$ to be the energy bins near the main peak of $h(x)$ (aka, full-energy peak). For $x \in \supp(F) \backslash \supp(H)$, $\rho(x) = f(x)$ because the gamma rays must come from the background in these regions. Typically, $\supp(F) \backslash \supp(H)$ contains two (or more) intervals, therefore we know the value of $\rho(x)$ on either side of set $A$. Then $\rho(x) \ \forall x \in A$ can be estimated using linear interpolation \citep{knoll2010radiation, alamaniotis2013kernel}. 
The above procedure is illustrated in Figure \ref{fig:gamma_unfolding} and the acceptance function is chosen to be
\begin{equation}
    \alpha(x) = 
        \begin{cases}
            1 - \frac{\rho(x)}{f(x)},  & x \in A, \\
            1, & \text{o.w.}
        \end{cases}
\end{equation}
\begin{figure}[H]
    \centering
    \includegraphics[width=0.5\textwidth]{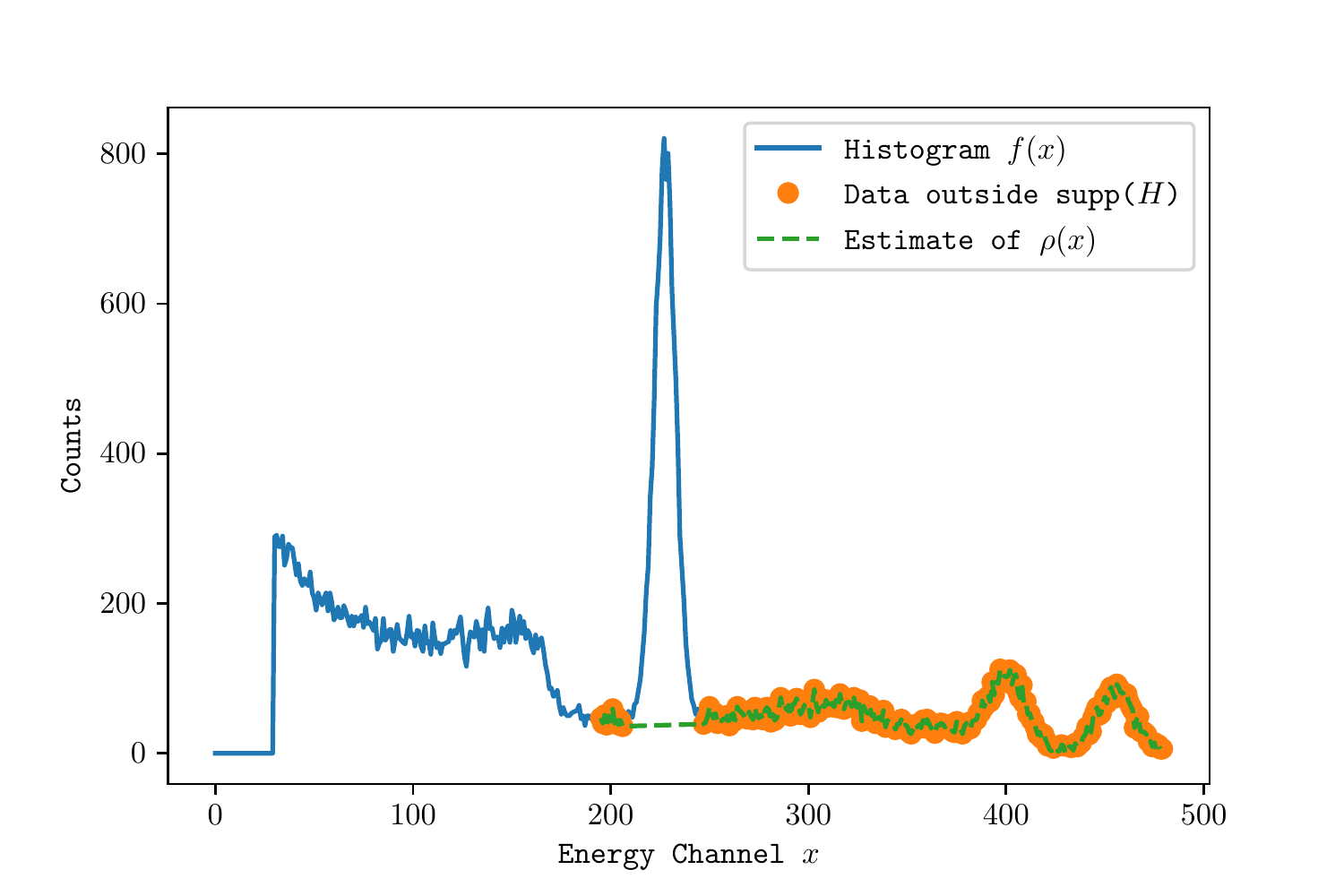}
    \caption{Visual illustration of how to estimate $\rho(x)$}
    \label{fig:gamma_unfolding}
\end{figure}

\section{Detailed Experimental Result in Sec. \ref{sec:experiment}}
\label{sec:more_experiment}

This section shows four tables corresponding to four experimental setups in Sec. \ref{sec:experiment}. 

\begin{table*}[h!]
    \caption{Average absolute estimation error on gamma ray spectra data, corresponding to the unfolding scenario. Several state-of-the-art MPE algorithms DPL \citep{ivanov2020dedpul}, EN \citep{elkan2008learning}, KM \citep{ramaswamy2016mixture} and TIcE \citep{bekker2018estimating} are selected. The mean absolute error ($\avg[\abs{\widehat{\kappa}^* - \kappa^*}]$) is reported, the smallest error among original, regrouping and subsampling versions is bolded. $(+/-)$ denotes that on average the estimator produces positive$/$negative estimation bias ($\sgn[\avg(\widehat{\kappa}^* - \kappa^*)]$). 
    }
    \label{table:gamma_err}
    \vskip 0.15in
    \centering
    \resizebox{\textwidth}{!}{
    \begin{tabular}{ c | c c c | c c c|c c c| c c c }
     $\kappa^*$ & DPL & ReDPL & SuDPL & EN & ReEN & SuEN & KM & ReKM & SuKM & TIcE & ReTIcE & SuTIcE \\ \hline
     $0.1$ & $0.063+$ & $0.053+$ & $\bm{0.008+}$ & $0.057+$ & $0.048+$ & $\bm{0.012-}$ & $0.226+$ & $0.146+$ & $\bm{0.036+}$ & $0.129+$ & $0.115+$ & $\bm{0.019+}$ \\ \hline
     $0.25$ & $0.055+$ & $\bm{0.008+}$ & $0.011+$ & $0.040+$ & $\bm{0.004-}$ & $0.017-$ & $0.216+$ & $\bm{0.020+}$ & $0.059+$ & $0.112+$ & $0.065+$ & $\bm{0.017+}$ \\ \hline 
     $0.5$ & $0.029+$ & $0.123-$ & $\bm{0.019-}$ &  $\bm{0.017+}$ & $0.129-$ & $0.032-$ & $0.130+$ & $\bm{0.043+}$ & $0.049+$ & $0.082+$ & $\bm{0.018+}$ & $0.019+$ \\ \hline 
     $0.75$ & $0.034+$ & $0.284-$ & $\bm{0.015+}$ & $\bm{0.022-}$ & $0.289-$ & $0.047-$ & $0.078+$ & $0.094-$ & $\bm{0.022+}$ & $0.055+$ & $0.045-$ & $\bm{0.021+}$ \\ \hline 
     average & $0.045+$ & $0.117-$ & $\bm{0.013+}$ & $0.034+$ & $0.118-$ & $\bm{0.027-}$ & $0.163+$ & $0.076+$ & $\bm{0.042+}$ & $0.095+$ & $0.061+$ & $\bm{0.019+}$ \\
    \end{tabular}
    }
    \vskip -0.1in
    
\end{table*}

\begin{table*}[h!]
    \caption{Average absolute estimation error on synthetic data, corresponding to domain adaptation scenario. Several state-of-the-art MPE algorithms DPL \citep{ivanov2020dedpul}, EN \citep{elkan2008learning}, KM \citep{ramaswamy2016mixture} and TIcE \citep{bekker2018estimating} are selected. The mean absolute error ($\avg[\abs{\widehat{\kappa}^* - \kappa^*}]$) is reported, and the smallest error among the original, regrouping and subsampling versions is bolded. $(+/-)$ denotes that on average the estimator produces positive$/$negative estimation bias ($\sgn[\avg(\widehat{\kappa}^* - \kappa^*)]$). 
    }
    \label{table:da_syn_err}
    \vskip 0.15in
    \centering
    
    \resizebox{\textwidth}{!}{
    \begin{tabular}{ c | c c c|c c c| c c c | c c c }
     $\kappa^*$ & DPL & ReDPL & SuDPL & EN & ReEN & SuEN & KM & ReKM & SuKM & TIcE & ReTIcE & SuTIcE \\ \hline
     $0.1$ & $0.089+$ & $0.075+$ & $\bm{0.014-}$ & $0.059+$ & $0.051+$ & $\bm{0.027-}$ & $0.102+$ & $0.091+$ & $\bm{0.013-}$ & $0.150+$ & $0.140+$ & $\bm{0.038+}$   \\ \hline 
     $0.25$ & $0.083+$ & $0.060+$ & $\bm{0.016-}$ & $0.037+$ & $\bm{0.016+}$ & $0.051-$ & $0.081+$ & $0.059+$ & $\bm{0.018-}$ & $0.137+$ & $0.108+$ & $\bm{0.040+}$  \\ \hline 
     $0.5$ & $0.053+$ & $0.028+$ & $\bm{0.024-}$ & $\bm{0.022-}$ & $0.057-$ & $0.077-$ & $0.052+$ & $0.037+$ & $\bm{0.020-}$ & $0.114+$ & $0.068+$ & $\bm{0.035+}$ \\ \hline
     $0.75$ & $\bm{0.016-}$ & $0.050-$ & $0.056-$ & $\bm{0.063-}$ & $0.118-$ & $0.114-$ & $\bm{0.018+}$ & $0.050-$ & $0.036-$ & $0.112+$ & $0.058+$ & $\bm{0.051+}$\\  \hline
     average & $0.060+$ & $0.053+$ & $\bm{0.028-}$ & $\bm{0.045-}$ & $0.061-$ & $0.067-$ & $0.063+$ & $0.059+$ & $\bm{0.022-}$ & $0.128+$ & $0.094+$ & $\bm{0.041+}$
    \end{tabular}
    }
    \vskip -0.1in

\end{table*}

\begin{table*}[h!]
    \caption{Average absolute estimation error on benchmark data, corresponding to domain adaptation scenario. Several state-of-the-art MPE algorithms DPL \citep{ivanov2020dedpul}, EN \citep{elkan2008learning}, KM \citep{ramaswamy2016mixture} and TIcE \citep{bekker2018estimating} are selected. The mean absolute error is reported, and the smallest error among the original, regrouping and subsampling versions is bolded. $+/-/\cdot$ denotes that on average the estimator produces positive$/$negative/no estimation bias.
    }
    \label{table:da_bench_err}
    \vskip 0.15in
    \centering
    \resizebox{\textwidth}{!}{
    \begin{tabular}{ c | c | c c c|c c c| c c c | c c c }
     Dataset & $\kappa^*$ & DPL & ReDPL & SuDPL & EN & ReEN & SuEN & KM & ReKM & SuKM & TIcE & ReTIcE & SuTIcE \\ \hline
     \multirow{4}{4em}{Mushroom} & $0.1$ & $0.075+$ & $0.069+$ & $\bm{0.053+}$ & $0.121+$ & $0.105+$ & $\bm{0.099+}$ & $0.061+$ & $0.054+$ & $\bm{0.039+}$ & $0.082+$ & $0.086+$ & $\bm{0.064+}$  \\ 
      & $0.25$ & $0.063+$ & $\bm{0.023+}$ & $0.040+$ & $0.139+$ & $\bm{0.101+}$ & $0.109+$ & $0.053+$ & $0.041+$ & $\bm{0.024+}$ & $0.057+$ & $0.034+$ & $\bm{0.024+}$ \\  
      & $0.5$ & $0.035+$ & $0.084-$ & $\bm{0.025+}$ & $0.132+$ & $\bm{0.074+}$ & $0.112+$ & $\bm{0.057-}$ & $0.189-$ & $0.063-$ & $\bm{0.026+}$ & $0.058-$ & $0.028-$ \\  
      & $0.75$ & $0.015-$ & $0.149-$ & $\bm{0.013-}$ & $0.097+$ & $\bm{0.033+}$ & $0.082+$ & $\bm{0.096-}$ & $0.253-$ & $0.108-$ & $0.076-$ & $0.134-$ & $\bm{0.027-}$ \\ \cline{2-14}
      & avg & $0.047+$ & $0.081-$ & $\bm{0.033+}$ & $0.122+$ & $\bm{0.078+}$ & $0.101+$ & $0.067-$ & $0.134-$ & $\bm{0.059-}$ & $0.060+$ & $0.078-$ & $\bm{0.036+}$\\ \hline
      
      \multirow{4}{4em}{Landsat}  & $0.1$ & $0.066+$ & $0.053+$ & $\bm{0.014+}$ & $0.168+$ & $0.139+$ & $\bm{0.115+}$ & $0.065+$ & $0.047+$ & $\bm{0.017+}$ & $0.070+$ & $0.074+$ & $\bm{0.031+}$  \\ 
      
      & $0.25$ & $0.060+$ & $0.010-$ & $\bm{0.005+}$ & $0.161+$ & $0.119+$ & $\bm{0.100+}$ & $0.053+$ & $0.027+$ & $\bm{0.007+}$ & $0.034+$ & $0.028+$ & $\bm{0.012-}$ \\  
      
      & $0.5$ & $0.037+$ & $0.034-$ & $\bm{0.010-}$ & $0.131+$ & $0.097+$ & $\bm{0.084+}$ & $0.037+$ & $0.018+$ & $\bm{0.016-}$ & $0.033+$ & $0.031-$ & $\bm{0.030-}$ \\  
      
      & $0.75$ & $\bm{0.022+}$ & $0.071-$ & $0.037-$ & $0.102+$ & $0.085+$ & $\bm{0.041+}$ & $0.028+$ & $\bm{0.023-}$ & $0.024-$ & $\bm{0.036-}$ & $0.042-$ & $0.067-$ \\ \cline{2-14}
      & avg & $0.046+$ & $0.042-$ & $\bm{0.017-}$ & $0.141+$ & $0.110+$ & $\bm{0.085+}$ & $0.046+$ & $0.029+$ & $\bm{0.016-}$ & $0.043+$ & $0.044-$ & $\bm{0.035-}$ \\ \hline
      
      \multirow{4}{4em}{Shuttle} & $0.1$ & $0.055+$ & $0.048+$ & $\bm{0.006+}$ & $0.105+$ & $0.096+$ & $\bm{0.053+}$ & $0.044+$ & $0.033+$ & $\bm{0.006-}$ & $0.083+$ & $0.072+$ & $\bm{0.023+}$  \\ 
      
      & $0.25$ & $0.045+$ & $0.010-$ & $\bm{0.007-}$ & $0.098+$ & $0.082+$ & $\bm{0.038+}$ & $0.027+$ & $0.018-$ & $\bm{0.016-}$ & $0.074+$ & $0.030+$ & $\bm{0.020+}$ \\  
      
      & $0.5$ & $0.025+$ & $0.102-$ & $\bm{0.016-}$ & $0.093+$ & $0.066+$ & $\bm{0.050+}$ & $\bm{0.017-}$ & $0.099-$ & $0.030-$ & $0.074+$ & $0.065-$ & $\bm{0.041+}$ \\  
      
      & $0.75$ & $\bm{0.022-}$ & $0.393-$ & $0.033-$ & $0.063+$ & $\bm{0.041+}$ & $0.044+$ & $\bm{0.058-}$ & $0.294-$ & $0.075-$ & $0.089+$ & $0.179-$ & $\bm{0.068+}$ \\ \cline{2-14}
      
      & avg & $0.037+$ & $0.138-$ & $\bm{0.015-}$ & $0.090+$ & $0.071+$ & $\bm{0.046+}$ & $0.036+$ & $0.111-$ & $\bm{0.032-}$ & $0.080+$ & $0.086-$ & $\bm{0.038+}$ \\ \hline
      
      \multirow{4}{4em}{MNIST17} & $0.1$ & $0.076+$ & $0.078+$ & $\bm{0.035+}$ & $0.184+$ & $0.157+$ & $\bm{0.116+}$ & $0.065+$ & $0.057+$ & $\bm{0.025+}$ & $0.101+$ & $0.093+$ & $\bm{0.052+}$  \\ 
      
      & $0.25$ & $0.058+$ & $0.047+$ & $\bm{0.008-}$ & $0.218+$ & $0.175+$ & $\bm{0.138+}$ & $0.053+$ & $0.017+$ & $\bm{0.009-}$ & $0.098+$ & $0.084+$ & $\bm{0.031+}$ \\  
      
      & $0.5$ & $0.032+$ & $0.047-$ & $\bm{0.022-}$ & $0.275+$ & $\bm{0.180+}$ & $0.191+$ & $0.029+$ & $0.030-$ & $\bm{0.017-}$ & $0.080+$ & $\bm{0.021+}$ & $0.047+$ \\  
      
      & $0.75$ & $\bm{0.023-}$ & $0.169-$ & $0.046-$ & $0.250+$ & $\bm{0.189+}$ & $0.217+$ & $\bm{0.016+}$ & $0.146-$ & $\bm{0.016-}$ & $0.081+$ & $0.094-$ & $\bm{0.060+}$ \\ \cline{2-14}
      
      & avg & $0.047+$ & $0.085-$ & $\bm{0.028-}$ & $0.231+$ & $0.175+$ & $\bm{0.166+}$ & $0.041+$ & $0.063-$ & $\bm{0.017-}$ & $0.090+$ & $0.073-$ & $\bm{0.048+}$ \\ \hline
      
      \multirow{1}{4em}{Overall}  & avg & $0.044+$ & $0.087-$ & $\bm{0.023-}$  & $0.146+$ & $0.109+$ & $\bm{0.099+}$ & $0.047+$ & $0.084-$ & $\bm{0.031-}$ & $0.068+$ & $0.070-$ & $\bm{0.039+}$ \\  \hline
    \end{tabular}
    }
    \vskip -0.1in

\end{table*}

\begin{table*}[h!]
    \caption{Average absolute estimation error on benchmark data, corresponding to the selected/reported at random scenario. Several state-of-the-art MPE algorithms DPL \citep{ivanov2020dedpul}, EN \citep{elkan2008learning}, KM \citep{ramaswamy2016mixture} and TIcE \citep{bekker2018estimating} are selected. The mean absolute error is reported, the smallest error among the original, regrouping and subsampling versions is bolded. $+/-/\cdot$ denotes that on average the estimator produces positive$/$negative/no estimation bias. 
    }
    \label{table:rar_bench_err}
    \vskip 0.15in
    \centering
    \resizebox{\textwidth}{!}{
    \begin{tabular}{ c | c | c c c|c c c| c c c | c c c }
     Dataset & $\kappa^*$ & DPL & ReDPL & SuDPL & EN & ReEN & SuEN & KM & ReKM & SuKM & TIcE & ReTIcE & SuTIcE \\ \hline
     \multirow{4}{4em}{Mushroom} & $0.1$ & $0.073+$ & $0.067+$ & $\bm{0.058+}$ & $0.119+$ & $0.105+$ & $\bm{0.099+}$ & $0.059+$ & $0.051+$ & $\bm{0.040+}$ & $0.084+$ & $0.082+$ & $\bm{0.057+}$  \\ 
      & $0.25$ & $0.060+$ & $0.023+$ & $\bm{0.005-}$ & $0.134+$ & $0.096+$ & $\bm{0.054+}$ & $0.054+$ & $0.034+$ & $\bm{0.010-}$ & $0.049+$ & $0.028+$ & $\bm{0.008+}$ \\  
      & $0.5$ & $0.032+$ & $0.087-$ & $\bm{0.018-}$ & $0.125+$ & $\bm{0.068+}$ & $0.076+$ & $\bm{0.073-}$ & $0.190-$ & $0.082-$ & $\bm{0.027+}$ & $0.057-$ & $0.052-$ \\  
      & $0.75$ & $\bm{0.023-}$ & $0.203-$ & $0.028-$ & $0.096+$ & $\bm{0.030+}$ & $0.066+$ & $\bm{0.100-}$ & $0.262-$ & $0.123-$ & $\bm{0.028-}$ & $0.096-$ & $0.058-$ \\ \cline{2-14}
      & avg & $0.047+$ & $0.095-$ & $\bm{0.027+}$ & $0.119+$ & $0.075+$ & $\bm{0.074+}$ & $0.072-$ & $0.134-$ & $\bm{0.064-}$ & $0.047+$ & $0.066-$ & $\bm{0.044-}$\\ \hline
      \multirow{4}{4em}{Landsat}  & $0.1$ & $0.066+$ & $0.053+$ & $\bm{0.034+}$ & $0.167+$ & $0.141+$ & $\bm{0.121+}$ & $0.066+$ & $0.049+$ & $\bm{0.028+}$ & $0.068+$ & $0.074+$ & $\bm{0.036+}$  \\ 
      & $0.25$ & $0.052+$ & $0.011-$ & $\bm{0.007\cdot}$ & $0.159+$ & $0.115+$ & $\bm{0.085+}$ & $0.061+$ & $0.035+$ & $\bm{0.011+}$ & $0.034+$ & $0.020+$ & $\bm{0.011-}$ \\  
      & $0.5$ & $0.043+$ & $0.065-$ & $\bm{0.011-}$ & $0.137+$ & $0.103+$ & $\bm{0.089+}$ & $0.033+$ & $0.020+$ & $\bm{0.016-}$ & $0.040+$ & $\bm{0.029-}$ & $0.048-$ \\  
      & $0.75$ & $0.032+$ & $0.097-$ & $\bm{0.023-}$ & $0.103+$ & $\bm{0.071+}$ & $\bm{0.071+}$ & $0.023+$ & $0.027-$ & $\bm{0.017\cdot}$ & $\bm{0.041-}$ & $0.087-$ & $0.068-$ \\ \cline{2-14}
      & avg & $0.048+$ & $0.057-$ & $\bm{0.019-}$ & $0.142+$ & $0.108+$ & $\bm{0.092+}$ & $0.046+$ & $0.033+$ & $\bm{0.018+}$ & $0.046+$ & $0.053-$ & $\bm{0.041-}$ \\ \hline
      \multirow{4}{4em}{Shuttle} & $0.1$ & $0.056+$ & $0.048+$ & $\bm{0.015+}$ & $0.112+$ & $0.097+$ & $\bm{0.060+}$ & $0.049+$ & $0.035+$ & $\bm{0.016-}$ & $0.080+$ & $0.066+$ & $\bm{0.032+}$  \\ 
      & $0.25$ & $0.044+$ & $0.018-$ & $\bm{0.007-}$ & $0.106+$ & $0.083+$ & $\bm{0.045+}$ & $0.029+$ & $0.020-$ & $\bm{0.018-}$ & $0.077+$ & $0.024+$ & $\bm{0.037+}$ \\  
      & $0.5$ & $0.026+$ & $0.162-$ & $\bm{0.008-}$ & $0.098+$ & $0.081+$ & $\bm{0.068+}$ & $\bm{0.016-}$ & $0.165-$ & $0.051-$ & $0.083+$ & $0.099-$ & $\bm{0.052+}$ \\  
      & $0.75$ & $\bm{0.012+}$ & $0.348-$ & $0.023-$ & $0.065+$ & $0.029+$ & $\bm{0.050+}$ & $\bm{0.072-}$ & $0.296-$ & $0.089-$ & $0.075+$ & $0.169-$ & $\bm{0.078+}$ \\ \cline{2-14}
      & avg & $0.035+$ & $0.144-$ & $\bm{0.013-}$ & $0.095+$ & $0.073+$ & $\bm{0.056+}$ & $\bm{0.042+}$ & $0.129-$ & $0.044-$ & $0.079+$ & $0.089-$ & $\bm{0.050+}$ \\ \hline
      \multirow{4}{4em}{MNIST17} & $0.1$ & $0.077+$ & $0.078+$ & $\bm{0.055+}$ & $0.183+$ & $0.157+$ & $\bm{0.134+}$ & $0.060+$ & $0.052+$ & $\bm{0.046+}$ & $0.096+$ & $0.083+$ & $\bm{0.076+}$  \\ 
      & $0.25$ & $0.063+$ & $0.052+$ & $\bm{0.005+}$ & $0.229+$ & $0.179+$ & $\bm{0.113+}$ & $0.060+$ & $0.027+$ & $\bm{0.009-}$ & $0.099+$ & $0.087+$ & $\bm{0.034+}$ \\  
      & $0.5$ & $0.035+$ & $0.042-$ & $\bm{0.014+}$ & $0.300+$ & $\bm{0.198+}$ & $0.208+$ & $0.022+$ & $0.024-$ & $\bm{0.011-}$ & $0.099+$ & $\bm{0.028+}$ & $0.053+$ \\  
      & $0.75$ & $\bm{0.014-}$ & $0.178-$ & $0.033-$ & $0.248+$ & $\bm{0.161+}$ & $0.200+$ & $\bm{0.008+}$ & $0.153-$ & $0.021-$ & $0.090+$ & $0.095-$ & $\bm{0.066+}$ \\ \cline{2-14}
      & avg & $0.047+$ & $0.088-$ & $\bm{0.027+}$ & $0.240+$ & $0.173+$ & $\bm{0.164+}$ & $0.038+$ & $0.064-$ & $\bm{0.022+}$ & $0.096+$ & $0.073+$ & $\bm{0.057+}$ \\ \hline
      \multirow{1}{4em}{Overall}  & avg & $0.044+$ & $0.096-$ & $\bm{0.022-}$  & $0.149+$ & $0.107+$ & $\bm{0.097+}$ & $0.050+$ & $0.090-$ & $\bm{0.037-}$ & $0.067+$ & $0.070+$ & $\bm{0.048+}$ \\  \hline
    \end{tabular}
    }
    \vskip -0.1in
\end{table*}

\newpage
\section{When Irreducibility Holds}
\label{app:exp_irred}
In theory, when irreducibility holds, baseline MPE methods shall be asymptotically unbiased estimators of the mixture proportion $\kappa^*$, regrouping may introduce negative bias, subsampling should not introduce bias.
Here we run some synthetically generated experiments (in a controlled setting) to verify the theoretical claim.

We specify the distributions $H$, $G$ and $F$  as: 
\begin{align*}
    H & \sim \mathcal{N} (\mu_1 = 0, \sigma_1 = 1) \\
    G & \sim  \mathcal{N} (\mu_2 = 2, \sigma_2 = 1) \\
    F & = (1-\kappa^*) G + \kappa^* H.
\end{align*}
where $G$ is irreducible w.r.t. $H$. We draw $m=500, n=1500$ instances from both $H$ and $F$. The true mixture proportion $\kappa^* = \kappa(F|H)$ is varied in $\left\{0.1, 0.25, 0.5, 0.75 \right\}$.

We draw $2000$ labeled instances from $F$. A $1$ hidden layer neural network with $16$ neurons was trained to predict $P(Y=1|X=x)$. The acceptance function $\alpha(x)$ used in Subsampling MPE is chosen to be 
\begin{equation*}
    \alpha(x) = \begin{cases}
        \widehat{P}(Y=1|X=x), & x \in A, \\
        1, & \text{o.w.},
    \end{cases}
\end{equation*}
where $A = \{x: \widehat{P}(Y=1|X=x) > 0.6 \} $. As for ReMPE, $10 \%$ of the sample from $F$ is chosen to be regrouped to the sample from $H$, as suggested by \citet{yao2022rethinking}. The above procedure was repeated $10$ times with different random seeds. Results where shown in Table \ref{table:irred_syn_avg}, where SuMPE never performs the worst 
\footnote{ 
In some cases, subsampling slightly increases the bias compared to original version, this is partly due to $P(Y|X)$ being estimated.
} and ReMPE may introduce extremely high bias.

\begin{table*}[h!]
    \caption{Bias of the estimators on synthetic data. Several state-of-the-art MPE algorithms DPL, EN, KM and TIcE are selected. For each of them, we compute the bias ($\avg[\widehat{\kappa}^* - \kappa^*]$) of the original, regrouping and subsampling version. The biggest absolute bias among the original, regrouping and subsampling versions is in bold.
    }
    \label{table:irred_syn_avg}
    \vskip 0.15in
    \centering
    
    \resizebox{\textwidth}{!}{
    \begin{tabular}{ c | c c c|c c c| c c c | c c c }
     $\kappa^*$ & DPL & ReDPL & SuDPL & EN & ReEN & SuEN & KM & ReKM & SuKM & TIcE & ReTIcE & SuTIcE \\ \hline
     $0.1$ & $\bm{+0.029}$ & $+0.021$ & $+0.010$ & $\bm{+0.020}$ & $+0.005$ & $+0.008$ & $\bm{+0.020}$ & $+0.007$ & $+0.006$ & ${+0.085}$ & $\bm{+0.089}$ & ${+0.063}$ \\ \hline 
     $0.25$ & $+0.012$ & $\bm{-0.062}$ & $-0.017$ & $-0.023$ &   $\bm{-0.080}$ & $-0.043$ & $-0.016$ & $\bm{-0.074}$ & $-0.031$ &   $\bm{+0.049}$  &  $+0.010$  & $+0.024$ \\ \hline 
     $0.5$ & $+0.038$ &  $\bm{-0.128}$ & $-0.003$ & $-0.045$ &  $\bm{-0.156}$ & ${-0.067}$ & $-0.001$ & $\bm{-0.133}$ & $-0.018$ &    $\bm{+0.059}$ &  $-0.048$ & $+0.032$ \\ \hline 
     $0.75$ & $+0.012$ & $\bm{-0.282}$ & $-0.015$ & ${-0.071}$ & $\bm{-0.278}$ &  ${-0.088}$ & $0.000$ &  $\bm{-0.278}$ &  $-0.012$ &   ${+0.089}$  & $\bm{-0.199}$  &  $+0.037$
    \end{tabular}
    }
    \vskip -0.1in

\end{table*}

\end{document}